\documentclass[10pt,journal,compsoc]{IEEEtran}
\usepackage{amsmath}
\usepackage{amsfonts}
\usepackage{amsthm}
\usepackage{booktabs}
\usepackage{subcaption}
\usepackage{dsfont}
\usepackage{multirow}
\usepackage{enumerate}
\usepackage{pifont}
\usepackage{graphicx}
\usepackage{bbm}
\usepackage{color}
\usepackage{xcolor}
\usepackage{wrapfig}

\usepackage{algorithm}  
\usepackage{algorithmic} 

\usepackage{hyperref} 

\newif\ifupdate\updatefalse

\newtheorem{theorem}{Theorem}

\ifCLASSOPTIONcompsoc
  \usepackage[nocompress]{cite}
\else
  \usepackage{cite}
\fi
\ifCLASSINFOpdf
\else
\fi
\begin{document}
%
\title{Learning Complete Topology-Aware Correlations Between Relations for Inductive Link Prediction}
%
%
%
%

\author{Jie~Wang,~\IEEEmembership{Senior Member, ~IEEE,} 
        Hanzhu~Chen,~Qitan~Lv,~Zhihao Shi,~Jiajun~Chen, \\
        Huarui~He,~Hongtao~Xie,~Defu~Lian,~\IEEEmembership{Member,~IEEE,}~Enhong~Chen,~\IEEEmembership{Fellow,~IEEE,} and~Feng~Wu,~\IEEEmembership{Fellow,~IEEE}
\IEEEcompsocitemizethanks{
\IEEEcompsocthanksitem J. Wang, H. Chen, Q. Lv, J. Chen, H. He, H. Xie, D. Lian, E. Chen and F. Wu are with: a) CAS Key Laboratory of Technology in GIPAS, University of Science and Technology of China, Hefei 230027, China; b) Institute of Artificial Intelligence, Hefei Comprehensive National Science Center, Hefei 230091, China. 
E-mail: jiewangx@ustc.edu.cn, chenhz@mail.ustc.edu.cn, qitanlv@mail.ustc.edu.cn, zhihaoshi@mail.ustc.edu.cn,
jiajun98@mail.us
tc.edu.cn, huaruihe@mail.ustc.edu.cn, htxie@ustc.edu.cn, liandefu@ustc.
edu.cn, cheneh@ustc.
edu.cn, fengwu@ustc.edu.cn.
}

\thanks{Manuscript received June, 2023.}}

%
%

\markboth{IEEE TRANSACTIONS ON PATTERN ANALYSIS AND MACHINE INTELLIGENCE,VOL. XX, NO. X, JUNE~2023}%
{Shell \MakeLowercase{\textit{et al.}}: Bare Demo of IEEEtran.cls for Computer Society Journals}
%




\IEEEtitleabstractindextext{%
\begin{abstract}

Inductive link prediction---where entities during training and inference stages can be different---has shown great potential for completing evolving knowledge graphs in an entity-independent manner. Many popular methods mainly focus on modeling graph-level features, while the edge-level interactions---especially the \textit{semantic correlations} between relations---have been less explored. However, we notice a desirable property of semantic correlations between relations is that they are inherently edge-level and entity-independent. This implies the great potential of the semantic correlations for the entity-independent inductive link prediction task. Inspired by this observation, we propose a novel subgraph-based method, namely TACO, to model \textbf{T}opology-\textbf{A}ware \textbf{CO}rrelations between relations that are highly correlated to their topological structures within subgraphs. Specifically, we prove that semantic correlations between any two relations can be categorized into seven \textit{topological patterns}, and then proposes Relational Correlation Network (RCN) to learn the importance of each pattern. To further exploit the potential of RCN, we propose Complete Common Neighbor induced subgraph that can effectively preserve complete topological patterns within the subgraph. Extensive experiments demonstrate that TACO effectively unifies the graph-level information and edge-level interactions to jointly perform reasoning, leading to a superior performance over existing state-of-the-art methods for the inductive link prediction task.

\end{abstract}

}

\maketitle 

\IEEEdisplaynontitleabstractindextext

%
\IEEEpeerreviewmaketitle

\IEEEraisesectionheading{\section{Introduction}\label{sec:introduction}}

%
%
%
%

\IEEEPARstart{K} {nowledge} graphs organize human knowledge in the form of factual triples (head entity, relation, tail entity), and these graphs represent entities as nodes and relations as edges. Examples of knowledge graphs include WordNet \cite{wordnet}, Freebase \cite{freebase}, and DBPedia \cite{dbpedia}.
Recently, knowledge graphs have been widely used in natural language processing \cite{NLP-field}, question answering \cite{QA-field}, and recommendation systems \cite{RS-field}. 
However, real-world knowledge graphs confront the challenge of continuously emerging new entities, such as new users and products in e-commerce knowledge graphs or new molecules in biomedical knowledge graphs \cite{grail }. Moreover, knowledge graphs often suffer from incompleteness, i.e., some links are missing.

To address these challenges, extensive research efforts have been devoted to the inductive link prediction task \cite{ILP-Hai, OOKB, sketch, congrl}. Inductive link prediction aims to predict missing links between entities in knowledge graphs, where entities during the training and inference stages can be different.
Despite the importance of inductive link prediction in real-world applications, many existing knowledge graph completion methods focus on the transductive link prediction task \cite{transe, transr, complex, distmult}, which can only handle the entities seen during the training stage\cite{drum}.
Inductive link prediction is challenging because it requires models to predict missing relations between unseen entities during training. This means that models need to generalize what they have learned from training entities to unseen entities.


Many existing inductive link prediction methods focus on predicting missing links by explicitly learning the connected  and closed logical rule, i.e., a reasoning path between the target nodes. Rule learning based methods \cite{neural-lp, drum, rlvrl, AIME} observe co-occurrence patterns of relations based on reasoning paths to explicitly mine logical rules. They are inherently inductive as the learned rules are entity-independent and can naturally generalize to new entities.
Recently, GraIL \cite{grail} models graph-level features by reasoning over subgraph structure surrounding the target link in an entity-independent manner. CoMPILE \cite{compile} simultaneously updates relation and entity embeddings to enhance the interaction between entities and relations during the message-passing procedure.
However, these methods mainly focus on modeling graph-level features while the edge-level \textit{semantic correlations} between relations have been less explored.
The absence of edge-level interactions may hinder the performance of edge-level modeling, posing a significant challenge to accurate inductive link prediction.

 	\begin{figure*}[htbp]
		\centering
		\includegraphics[width=2\columnwidth]{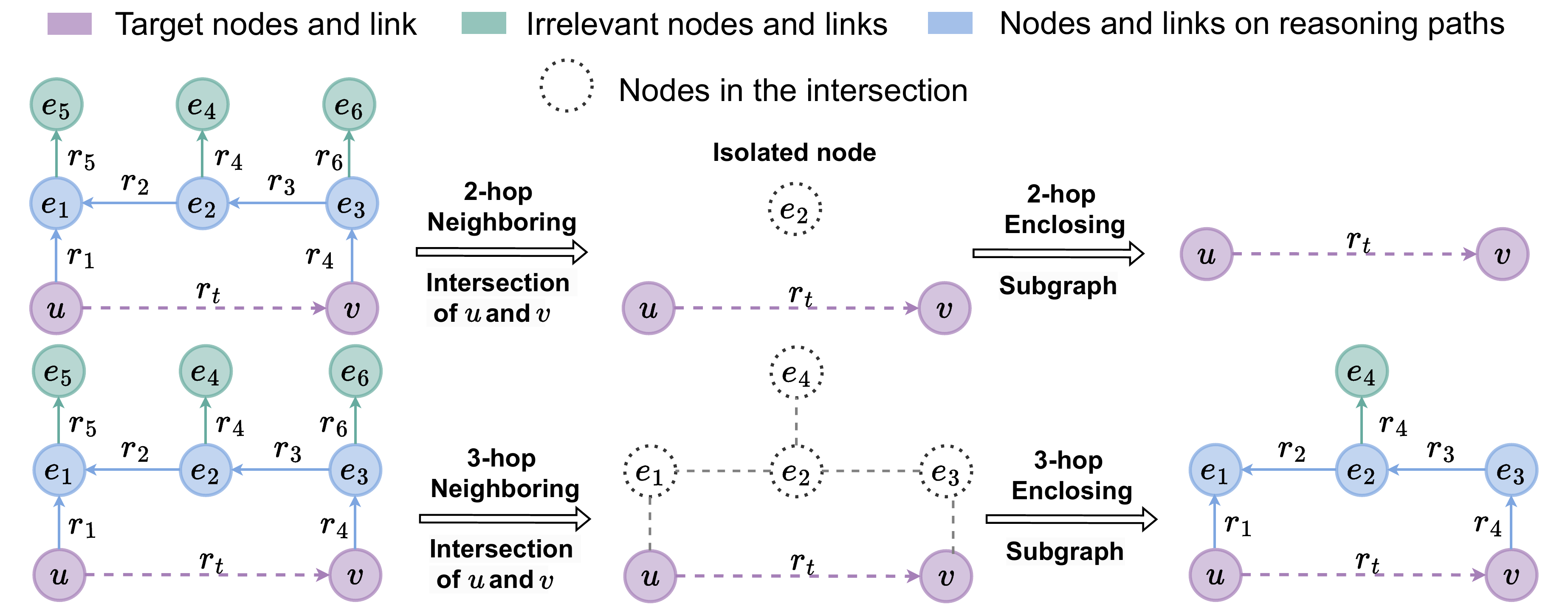}
		\caption{An example of enclosing subgraph extraction process. Enclosing subgraph contains links between nodes in the $n$-hop neighborhood intersection of target nodes.}
		\label{fig:intro_fig}
	\end{figure*}

In this paper, we exploit edge-level {semantic correlations} between relations that are commonly seen in knowledge graphs. 
For example, the relation pair in Freebase \cite{freebase}, ``/people/person/nationality" and ``/people/ethnicity/languages{\_}spoken" is strongly correlated as what kind of language spoken by a person is highly correlated to the nationality, while the correlation is weaker between the pair ``/people/person/nationality" and ``/film/film/country". Moreover, the topological structure between relations---the connected way for each relation pair---can
be different, which also has an impact on the correlation patterns. For example, considering the relation pair ``has\_gender"  and ``father\_of" in Figure \ref{fig:1}, they are connected by entity $e_1$ in a tail-to-tail manner and $e_2$ in a head-to-tail manner, which are different topological structures (see Section \ref{rcg} for rigorous definition).

	\begin{figure}[t]
		\centering 
		\includegraphics[width=1\columnwidth]{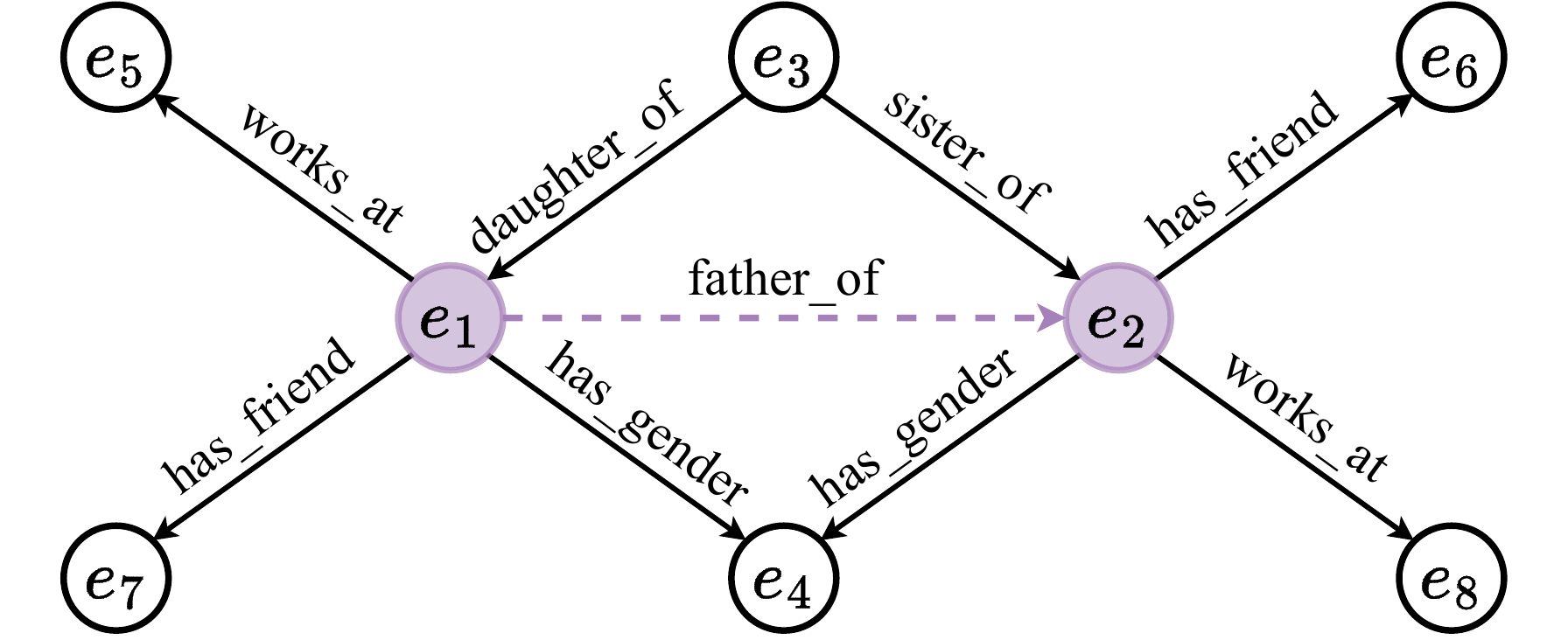}
		\caption{An example in knowledge graphs.}
		\label{fig:1} 
	\end{figure}

The aforementioned observation implies that semantic correlations between relations are highly correlated to their topological structures. Thus, to effectively leverage the \textbf{T}opology-\textbf{A}ware \textbf{C}orrela\textbf{T}ions between relations in knowledge graphs, we propose a novel subgraph-based inductive reasoning method, namely TACT. Specifically, TACT models semantic correlations between relations in two aspects: correlation patterns and correlation coefficients. We first categorize all relation pairs into \textit{seven} correlation patterns according to their topological structures (see Section \ref{rcg}) and convert the original knowledge graph into Relational Correlation Graph (RCG), where nodes represent the relations and edges indicate the correlation patterns between relations. 
Based on RCG, we further propose Relational Correlation Network (RCN) to learn the correlation coefficients of different correlation patterns.

An earlier version of TACT has been published at AAAI 2021 \cite{tact}. 
TACT predicts missing links on the widely used enclosing subgraph \cite{grail}, which contains relations in the $n$-hop neighborhood intersection of the target entities.
However, the enclosing subgraph hurts semantic correlations between relations, as it either drops relevant relations on reasoning paths or introduces irrelevant relations, as shown in Figure \ref{fig:intro_fig}. The $2$-hop enclosing subgraph induced by entities $u$ and $v$ cannot preserve the reasoning path that contains the isolated node $e_2$. The $3$-hop enclosing subgraph can preserve the reasoning path, while it introduces irrelevant relations $r_4$ and nodes $e_4$ as well.
From the perspective of rule mining, relations within reasoning paths are crucial to construct closed rules between target nodes \cite{drum}.
On the WN18RR dataset \cite{wn18rr}, over 45\% of the 2-hop enclosing subgraphs extract only the target relations, and over 75\% of the 3-hop enclosing subgraphs contain irrelevant relations (see Section \ref{3.5.2}). 
Therefore, the resulting RCG is inaccurate to extract relation correlations around the target link.




To tackle this problem, this journal manuscript significantly extends the conference version by an enhanced version of TACT, namely TACO. 
In this journal version, we propose a novel Complete Common Neighbor induced (CCN) subgraph extraction method to effectively preserve the complete reasoning paths. The key idea of CCN is to preserve all common neighbors \cite{common} between target nodes, especially the isolated nodes (e.g., node $e_2$ in Figure \ref{fig:intro_fig}) by introducing equivalent relations. Then, TACO learns complete topology-aware correlations between relations within CCN subgraphs.

Compared with the conference version, we have the following extensions.
First, we provide comprehensive statistical analysis to reveal the limitations of enclosing subgraphs in Section \ref{3.5.2}. 
Second, based on the analysis, we propose the CCN subgraph in Section \ref{da:pip} to effectively preserve the reasoning paths by adding equivalent relations from isolated nodes to target entities.
Third, to further leverage the complete topology-aware correlations between relations, we propose CCN+ subgraphs in Section \ref{ca:pip} to properly preserve the isolated nodes and complete relations on reasoning paths by traversing all relations in the reasoning path based on a recursive algorithm.
Finally, we conduct extensive ablation studies to demonstrate the effectiveness of different input embeddings and relation correlation patterns in Section \ref{ab: input emb} and Section \ref{ab: diff rc}, respectively. 
Experiments demonstrate that 
TACO effectively unifies graph-level information and edge-level interactions to jointly perform reasoning, leading to a superior performance over TACT and existing state-of-the-art methods on inductive link prediction benchmarks.



\section{Related Work}
\subsection{Rule Learning Based Methods}
Rule-based approaches are dedicated to mining connected and closed logical rules based on the observed highly frequent co-occurrence patterns of relations, which are inherently inductive as the learned rules are entity-independent. Mining rules from knowledge graphs is the central task of inductive logic programming \cite{ILP}. Traditional rule-based methods lack expressive power and suffer from the scalability to large knowledge graphs due to the rule-based nature \cite{grail}. Recently, Neural-LP \cite{neural-lp} proposes an end-to-end differentiable manner by using Tensorlog operators \cite{tensorlog} to learn the rule structure and parameters of logical rules simultaneously. Based on Neural-LP, DRUM \cite{drum} further proposes to mine more accurate rules in knowledge graphs from the perspective of low-rank matrix approximation. However, rule-based methods mainly focus on explicitly learning the first-order logical rules, which limits their ability to model more complex semantic correlations between relations and scalability to large knowledge graphs \cite{grail}.

 	\begin{figure*}[ht]
		\centering
		\includegraphics[width=2\columnwidth]{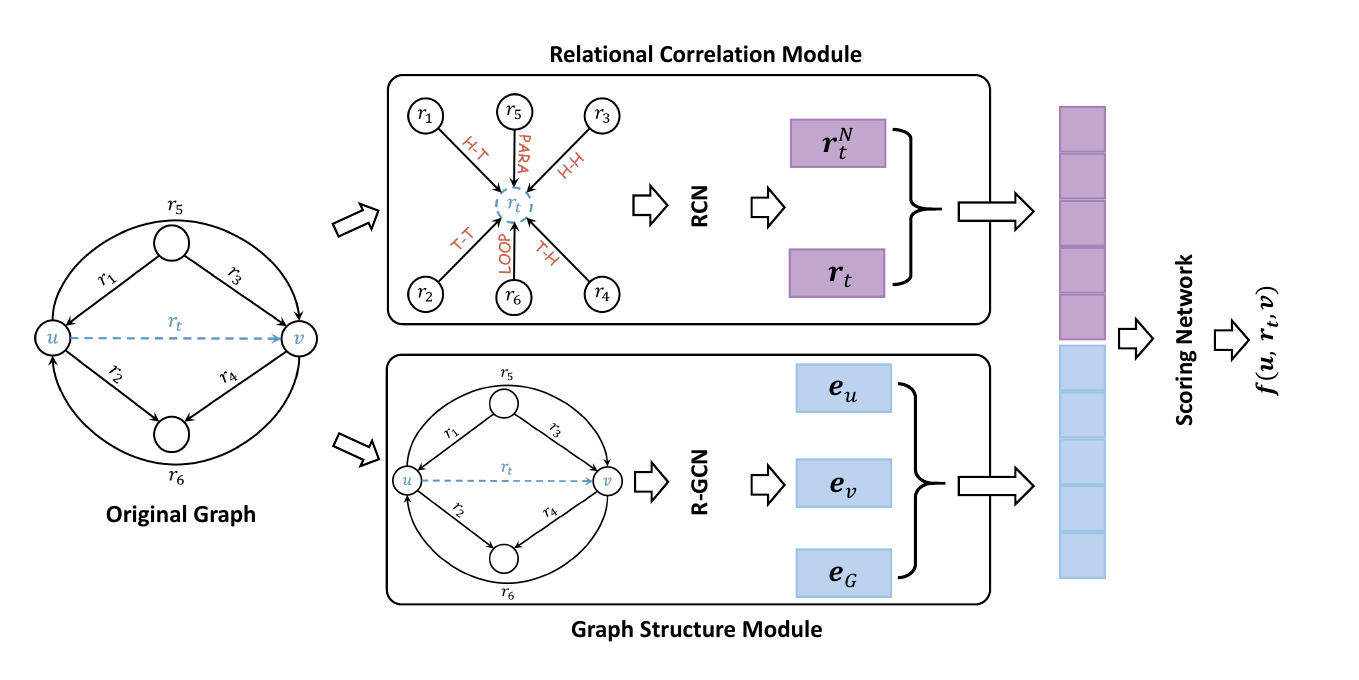}
		\caption{An overview of TACO. TACO consists of two modules: the relational correlation module and the graph structure module. We use a scoring network to score a triple based on the output of the two modules.}
		\label{fig:method-all}
	\end{figure*}

\subsection{Embedding Based Methods}
Embedding based methods primarily focus on learning the low-dimensional embeddings as representations to retrieve the relational information of the knowledge graph, which has been shown promising for knowledge graph completion \cite{rotate, DURA, HAKE, house}. Some embedding based methods can generate embeddings for unseen entities. LAN \cite{LAN} generates entity embeddings for unseen entities by aggregating known neighbor entity embeddings around the unseen entities with graph neural networks (GNNs)\cite{gnn-semi, xu2018powerful}. However, these methods require the unseen entities surrounded by the known entities, which restrains the ability to generalize to new knowledge graphs, where entities in two graphs have no overlap. 

Recently, GraIL \cite{grail} proposes a new link prediction framework based on GNN that reasons on subgraph structures, which can conduct link prediction in an entity-independent manner. Based on GraIL, CoMPILE \cite{compile} simultaneously updates relation and entity embeddings from both directed and undirected subgraph structures to enhance the interaction between entities and relations during the message passing procedure. SNRI\cite{SNRI} extracts enclosing subgraphs with neighboring infomax to learn the subgraph representation.
However, these methods do not take into account semantic correlations between relations, which are common in knowledge graphs. Meanwhile, the commonly
used enclosing subgraph cannot effectively preserve reasoning paths and eliminate irrelevant relations. As these methods reason on subgraphs, they can also be referred to as subgraph-based methods.
More recently, NBFNet \cite{nbfnet} views the connected rules from the target head entity to the target tail entity as paths and introduces a generalized neural Bellman-Ford network for link prediction.
Though NBFNet shares some similarities with subgraph-based methods, it is essentially different from them. NBFNet needs to reason over the whole graph for a test example, while subgraph-based methods only need to reason over a subgraph structure. Additionally, NBFNet benefits from a large number of negative sampling, while subgraph-based methods can reduce negative samples to one.

\subsection{Link Prediction with GNNs}
Recently, GNNs \cite{liu2021indigo, gat, graphformer, gin} have shown great potential in link prediction as knowledge graphs naturally have graph structures.  RGCN \cite{RGCN} proposes a relational graph neural network to take into account the connected relations when applying neighborhood aggregation on the entities. More recently, RGHAT \cite{RGHAT} proposes a relational graph neural network with hierarchical attention to effectively utilize the neighborhood information of entities in knowledge graphs. However, these methods have difficulty in generalizing to unseen entities as they rely on the learned entity embeddings during the training stage to predict the missing links.

\subsection{Modeling Correlations Between Relations}
Several existing knowledge graph embedding methods take into account the problem of modeling correlations between relations. TransF \cite{TransF} decomposes the relation–specific projection spaces into a small number of spanning bases, which are shared by all relations. TransCoRe \cite{COR-1} learns the embedded relation matrix by decomposing it as a product of two low-dimensional matrices. Different from the aforementioned works, our work

	\begin{enumerate}
		\item[(a)] classifies all relation pairs into seven \textit{topological patterns} and proposes a novel relational correlation network to model \textit{topology-aware} correlations.
        \item[(b)] proposes two novel Complete Common Neighbor induced subgraphs to alleviate the loss of relevant relations on reasoning paths.
	\item[(c)] considers the inductive link prediction task, while the mentioned knowledge graph embedding methods and GNN based methods have difficulty in handling this entity-independent task.
		\item[(d)] outperforms the existing state-of-the-art methods on benchmarks for the inductive link prediction task. 
	\end{enumerate}

\section{Methods}
In this section, we first introduce notations used in this paper in Section \ref{nota}. And then, we introduce our proposed method TACO. To perform inductive link prediction, TACO aims at scoring a given triple $(u, r_t, v)$ in an entity-independent manner, where $r_t$ is the target relation between the head entity $u$ and the tail entity $v$. Specifically, TACO consists of the relational correlation module and the graph structure module, which correspond to  Relational Correlation Network (RCN) in Section \ref{mcbr} and Relational Graph Correlation Network (R-GCN) in Section \ref{3.3}, respectively. RCN is proposed based on the observation that semantic correlations between relations are highly correlated to their topological structures, which are commonly seen in knowledge graphs. Moreover, we design R-GCN based on GraIL \cite{grail} to leverage the graph structure information. TACO organizes the two modules in a general framework to unify graph-level information and edge-level interactions. Figure \ref{fig:method-all} gives an overview of the proposed TACO. Finally,  we indicate the motivation of our Complete Common Neighbor induced(CCN) subgraph in Section \ref{3.5.1} and have a statistical analysis of the existing inductive datasets in Section \ref{3.5.2} to reveal the limitation of commonly used enclosing subgraph,. Further, we introduce 
the CCN subgraph methods based on the statistical analysis in Sections \ref{da:pip} and \ref{ca:pip}.

\subsection{Notations}\label{nota}
	Given a set $\mathcal{E}$ of entities and a set $\mathcal{R}$ of relations, a knowledge graph $\mathcal{G}=\{(u,r, v)|u,v\in\mathcal{E},\,r\in\mathcal{R}\}$ is a collection of factual triples, where $u$, $v$, and $r$ represent the head entities, the tail entities, and the relations between the head and tail entities, respectively. We use $\textbf{e}_u, \textbf{r}$ and $\textbf{e}_v$ to denote the embedding of the head entity, the relation and the tail entity. Let $d$ denote the embedding dimension of each entity and relation. We denote the $i^{th}$ entry of a vector $\textbf{e}$ as $[\textbf{e}]_i$. Let $\circ$ be the Hadamard product between two embedding vectors,
	\begin{align*}
		[\textbf{a} \circ \textbf{b}]_i = [\textbf{a}]_i \cdot [\textbf{b}]_i ,
	\end{align*}
	and we use $\oplus$ to denote the concatenation of vectors.
 
 \subsection{Modeling Correlations Between Relations}\label{mcbr}
 To model semantic correlations between relations, we consider the correlations in two aspects: 
 	\begin{enumerate}
		\item[(a)] Correlation patterns: The correlation patterns between any two relations are highly correlated to their topological structures in knowledge graphs. 
		\item[(b)] Correlation coefficients: The correlation coefficients can represent the degree of semantic correlations between any two relations. 
	\end{enumerate}
 \subsubsection{Relational Correlation Graph}\label{rcg}
 To model the correlation patterns between relations, we classify all relation pairs into seven different topological patterns. As illustrated in Figure \ref{fig:2}, the topological patterns are ``head-to-tail", ``tail-to-tail", ``head-to-head", ``tail-to-head", ``parallel", ``loop", and ``not connected". We define the corresponding correlation patterns as ``H-T", ``T-T", ``H-H", ``T-H", ``PARA", ``LOOP", and ``NC", respectively. For example, we denote $(r_1, \text{H-T}, r_2)$ as the correlation between relation $r_1$ and $r_2$ is the ``H-T" pattern for $r_2$, which indicates that $r_1$ and $r_2$ are connected in a head-to-tail manner. $(r_1, \text{PARA}, r_2)$ indicates that the two relations are connected by the same head entity and tail entity, and $(r_1, \text{LOOP}, r_2)$ indicates that the two relations form a loop structure in the original local graph. We prove that the number of topological patterns between any two relations is at most \textit{seven} in the Appendix \ref{appenA}.

 	Based on the definition of different correlation patterns, we can convert the original knowledge graph to Relational Correlation Graph (RCG), where the nodes represent the relations and the edges indicate the correlation patterns between any two relations in the original knowledge graph. Figure \ref{fig:2} shows the topological patterns between any two relations and the corresponding RCGs. Notice that for the topological pattern that two relations are not connected, its corresponding RCG consists of two isolated nodes.

	\begin{figure}[ht]
		\centering
		\includegraphics[width=1\columnwidth]{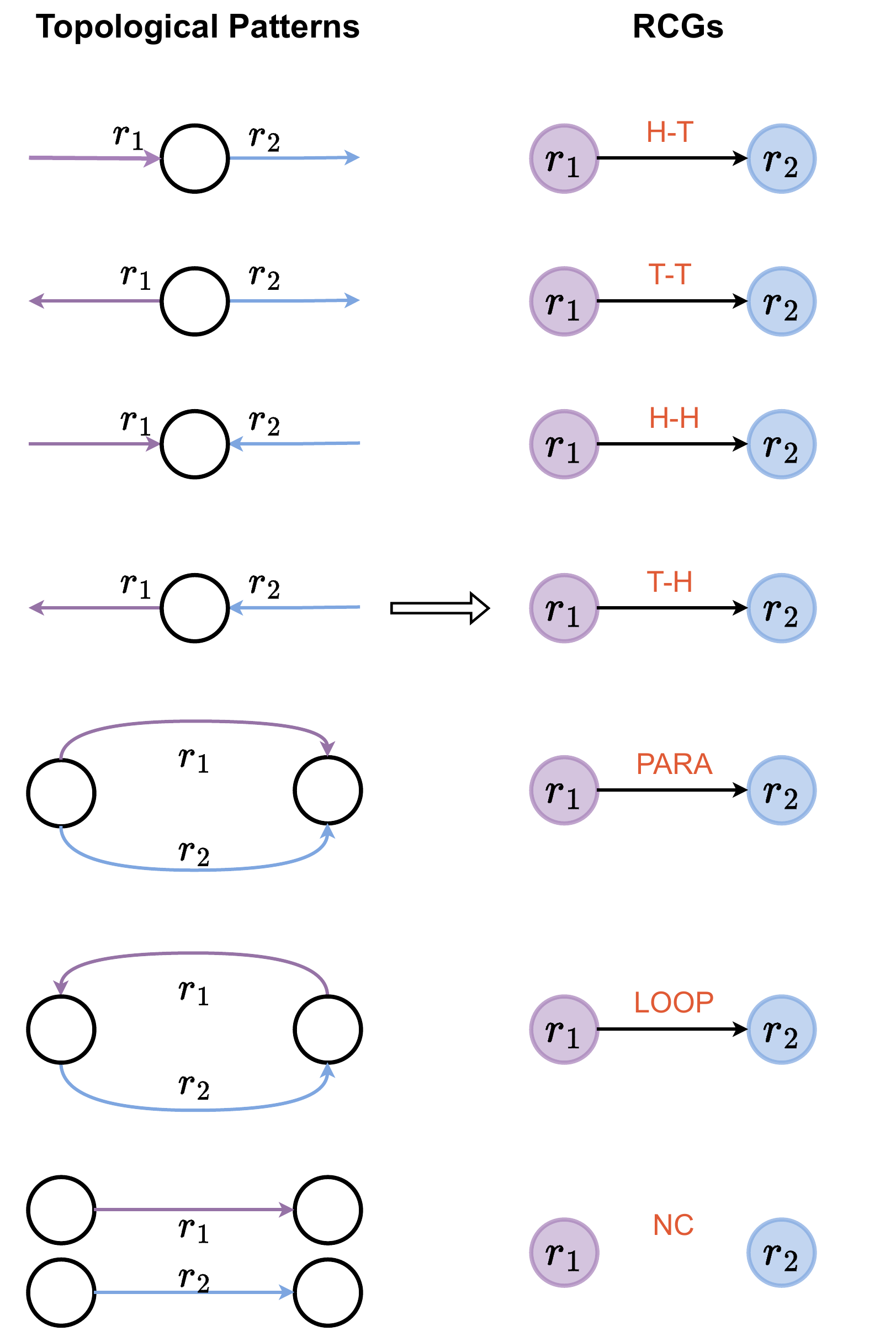}
		\caption{An illustration of the topological patterns between any two relations and the corresponding RCGs. For the topological pattern where two relations are not connected, its corresponding RCG consists of two isolated nodes.}
		\label{fig:2}
	\end{figure}
\subsubsection{Relational Correlation Network}\label{rcn}
Based on RCG, we propose Relational Correlation Network (RCN) to model the importance of different correlation patterns for inductive link prediction. The RCN module consists of two parts: the correlation pattern part and the correlation coefficient part. The correlation pattern part takes into account the influence of different topological structures between relations and the correlation coefficient part aims at learning the degree of different correlations between relations.

For an edge around the target relation $r_t$, we can divide all its adjacent edges in RCG into six connected categories by the topological patterns ``H-T", ``T-T", ``H-H", ``T-H", ``PARA", and ``LOOP", respectively. Notice that the topological pattern ``NC" is not considered as it means the edges (relations) are not connected in the original knowledge graph. For the six connected categories, we use six linear transformations to learn the different semantic correlations corresponding to the topological patterns. To better differentiate the degree of different correlations for the target relation $r_t$, we further use attention networks to learn corresponding correlation coefficients for all correlation categories.

        \begin{figure*}[ht]
    \centering 
\begin{subfigure}{2\columnwidth}
\centering
  \includegraphics[width=400pt]{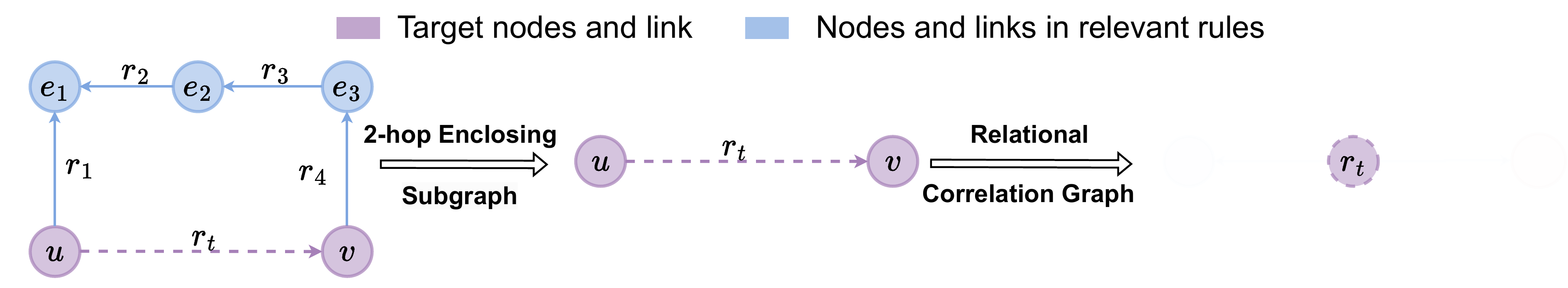}
  \label{fig:why-en}
  \caption{The $2$-hop enclosing subgraph and corresponding relational correlation graph via adapting the enclosing subgraph {method}.}\label{fig:en}
\end{subfigure}
\medskip
\begin{subfigure}{2\columnwidth}
\centering
  \includegraphics[width=400pt]{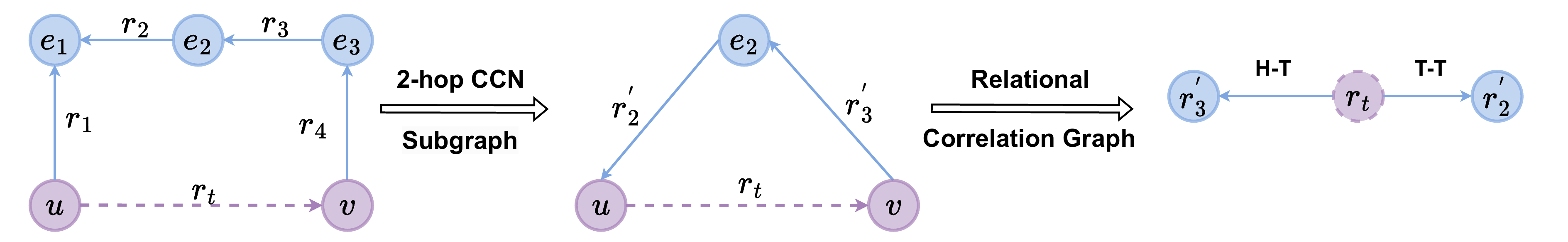}
  \label{fig:why-sa}
  \caption{The $2$-hop CCN subgraph and corresponding relational correlation graph via adapting the CCN subgraph {method}.}\label{fig:saen}
\end{subfigure}\hfil 
\medskip
\begin{subfigure}{2\columnwidth}
\centering
  \includegraphics[width=400pt]{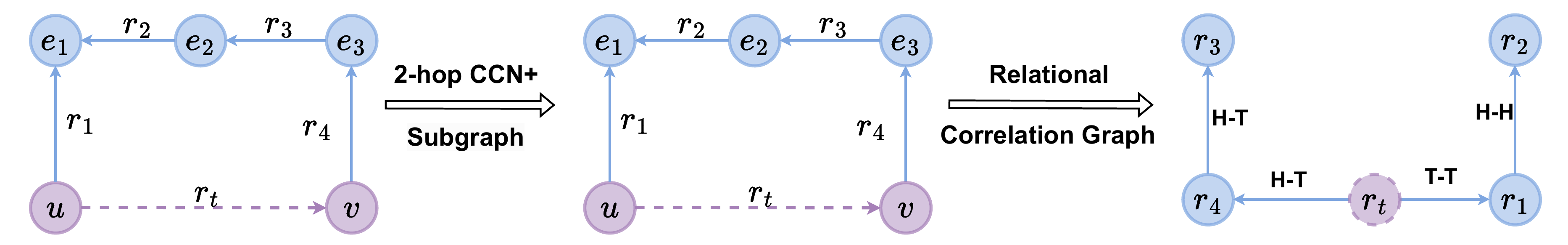}
  \label{fig:why-03}
  \caption{The $2$-hop CCN+ subgraph and corresponding relational correlation graph via adapting the CCN+ subgraph {method}.} \label{fig:03}
\end{subfigure}\hfil 
\vskip -0.15in
\caption{Comparison between the $2$-hop enclosing subgraph {method}, $2$-hop CCN {method}, and $2$-hop CCN+ method in the same original knowledge graph.}
\label{fig:illustration}
\vskip -0.1in
\end{figure*}

	Specifically, we aggregate all the correlation coefficients of different correlation patterns for the relation $r_t$ to get the neighborhood embedding in a local extracted subgraph, which is denoted by $\textbf{r}_t^{N}$. 
	\begin{align}\label{rel-agg}
		\textbf{r}_t^{N} = \frac{1}{6}\sum_{p=1}^{6} (\textbf{N}^p_t \circ \mathbf{\Lambda}^p_t) \textbf{R} \textbf{W}^p 
	\end{align}
	where $\textbf{W}^p \in \mathbb{R}^{d\times d}$ is the weight parameter matrix, $\textbf{R} \in \mathbb{R}^{|\mathcal{R}|\times d}$ denotes the embedding of all relations. Suppose the embedding of $r_i$ is $\textbf{r}_i \in \mathbb{R}^{1\times d}$, then $\textbf{R}_{[i,:]} = \textbf{r}_i$ where $\textbf{R}_{[i,:]}$ denotes the $i^{th}$ slice along the first dimension. $\textbf{N}^p_t \in \mathbb{R}^{1\times |\mathcal{R}|}$ is the indicator vector where the entry $[\textbf{N}^p_t]_i=1$ if $r_i$ and $r_t$ are connected in the $p^{th}$ topological pattern, otherwise $[\textbf{N}^p_t]_i=0$. $\mathbf{\Lambda}^p_t \in \mathbb{R}^{1\times|\mathcal{R}|}$ is the weight parameter, which indicates the degree of different correlations for the relation $r_t$ in the $p^{th}$ correlation pattern. Note that, we restrict $[\mathbf{\Lambda}_t^p]_i \ge 0$ and $\sum_{i=1}^{|\mathcal{R}|} [\mathbf{\Lambda}_t^p]_i = 1$. 

	Furthermore, we concatenate $\textbf{r}_t$ and $\textbf{r}_t^{N}$ to get the final embedding $\textbf{r}^F_t$.
	\begin{align}\label{rel-agg-2}
		\textbf{r}^F_t =\sigma ([\textbf{r}_t \oplus \textbf{r}_t^{N}] \textbf{H})
	\end{align}
	where $\textbf{H}\in \mathbb{R}^{2d\times d}$ is the weight parameters, and $\sigma$ is an activation function, such as $\text{ReLU}(\cdot)=\max(0, \cdot)$. We call the module that models semantic correlations between relations as the \textit{relational correlation module}, and $\textbf{r}^F_t$ is the final output of the module.

 	\subsection{Modeling graph structures}\label{3.3}
  	For a target prediction triple $(u, r_t, v)$, the local extracted graph around it contains the information about how the triple connected with its neighborhoods. To take advantage of the local graph structural information, we use a graph structural network to embed local graphs into vectors based on GraIL \cite{grail}. To model {the graph structure around the triple $(u, r_t, v)$,} we perform the following steps: (1) subgraph extraction; (2) node labeling; (3) graph embedding. Notice that nodes here in step 2 represent entities in the original knowledge graphs.

   \subsubsection{Subgraph Extraction}\label{sec:subex}

    For a triple $(u, r_t, v)$, we first extract the enclosing subgraph surrounding the target nodes $u$ and $v$ \cite{grail}. The following steps give the enclosing subgraph between nodes $u$ and $v$. First, we compute the neighbors $\mathcal{N}_k(u)$ and $\mathcal{N}_k(v)$ of the two nodes $u$ and $v$, respectively, where $k$ denotes the max distance of neighbors around node $u$ and $v$. Second, we take an intersection of $\mathcal{N}_k(u)$ and $\mathcal{N}_k(v)$ to get $\mathcal{N}_k(u) \cap \mathcal{N}_k(v)$. Third, we compute the enclosing subgraph $\mathcal{G}(u, r_t, v)$ by pruning nodes of $\mathcal{N}_k(u) \cap \mathcal{N}_k(v)$ that are isolated, i.e., the nodes have no edges with the other nodes in  $\mathcal{N}_k(u) \cap \mathcal{N}_k(v)$.

 

     \subsubsection{Node Labeling}

	Afterward, we label the surrounded nodes in the extracted enclosing subgraph following the approach proposed by GraIL \cite{grail}. We first label each node $i$ in the subgraph around nodes $u$ and node $v$ with the tuple $(d(i,u), d(i,v))$, where $d(i,u)$ denotes the shortest distance between nodes $i$ and $u$ without counting any path through $v$ (likewise for $d(i,v)$). This captures the topological position of each node with respect to the target nodes in the extracted enclosing subgraph. The two target nodes $u$ and $v$ are uniquely labeled $(0,1)$ and $(1,0)$. Then, the node features are correspondingly defined as $[\text{one-hot}(d(i,u)) \oplus \text{one-hot}(d(i,v))]$, where $\text{one-hot}(p) \in \mathbb{R}^{1\times d}$ represents the one-hot vector that only the $p^{th}$ entry is $1$, where $d$ represents the dimension of the node embeddings. 

    \subsubsection{Graph Embedding}
	After node labeling of the extracted enclosing subgraph, the nodes in the subgraph have the initial embeddings. We can then use the R-GCN module \cite{RGCN} to learn the node embeddings on the extracted enclosing subgraph $\mathcal{G}(u, r_t, v)$, which can be denoted as 
	\begin{align*}
		\textbf{e}_i^{(k+1)} = \sigma\left( \sum_{r\in \mathcal{R}} \sum_{j\in \mathcal{N}_i^r} \frac{1}{c_{i,r}} \textbf{e}_j^{(k)}  \textbf{W}_r^{(k)} + \textbf{e}_i^{(k)}\textbf{W}_0^{(k)} \right)
	\end{align*}
 	where $\textbf{e}_i^{(k)}$ denotes the embedding of entity $e_i$ of the $k^{th}$ layer in the R-GCN. $\mathcal{N}_i^r$ denotes the set of neighborhood indices of node $i$ under relation $r\in \mathcal{R}$. $c_{i,r} = |\mathcal{N}_i^r|$ is a normalization constant. $\textbf{W}_r^{(k)}\in\mathbb{R}^{d\times d} (r\in\mathcal{R})$ and $ \textbf{W}_0^{(k)} \in \mathbb{R}^{d\times d}$ are the weight parameters. $\sigma(\cdot)$ is a activation function, such as the $\text{ReLU}(\cdot) = \max(0, \cdot)$.

  	Suppose that the number of layers in the R-GCN module is $L$, we calculate the embedding of the whole enclosing subgraph $\mathcal{G}(u,r_t,v)$ as
	\begin{align*}
		\textbf{e}_{\mathcal{G}(u,r_t,v)}^{(L)} = \frac{1}{|\mathcal{V}_{\mathcal{G}(u,r_t,v)}|} \sum_{i\in \mathcal{V}_{\mathcal{G}(u,r_t,v)}} \textbf{e}_i^{(L)},
	\end{align*}
	where $\mathcal{V}_{\mathcal{G}(u,r_t,v)}$ denotes the set of nodes in graph $\mathcal{G}(u,r_t,v)$. We further combine the target nodes and the subgraph embedding, the structural information is represented by the vector $\textbf{e}_S \in \mathbb{R}^{1\times 3d}$,
	\begin{align*}
		\textbf{e}_S = \textbf{e}_{\mathcal{G}(u,r_t,v)}^{(L)} \oplus \textbf{e}_u^{(L)} \oplus \textbf{e}_v^{(L)}
	\end{align*}
	We call the module that can model the graph structures as the \textit{graph structure module}, and the embedding vector $\textbf{e}_S$ is the final output of the module.

        \subsection{The Framework of TACO}
        \subsubsection{Scoring Network}
	The relational correlation module and graph structure module output the embedding vectors $\textbf{r}_t^F$ and $\textbf{e}_S$, respectively. To organize the two modules in a unified framework, we design a scoring network to combine the outputs of the two modules and get the score for a given triple $(u, r_t, v)$. The score function $f(u, r_t, v)$ is defined as
 	\begin{align*}
		f(u,r_t,v) = [\textbf{r}_t^{F} \oplus \textbf{e}_S] \textbf{W}_S
	\end{align*}
	where $\textbf{W}_S \in \mathbb{R}^{4d\times 1}$ is the weight parameters.
        \subsubsection{Loss Function}
	We conduct negative sampling and train the model to score positive triples higher than the negative triples using a {noise-contrastive} hinge loss following TransE \cite{transe}. The loss function $\mathcal{L}$ is 
	\begin{align*}
		\mathcal{L} = \sum_{i\in [n], (u,r_t,v)\in \mathcal{G}} \max(0, f(u'_i,r'_{t,i},v'_{i}) - f(u,r_t,v) + \gamma)
	\end{align*}
	\noindent where $\gamma$ is the margin hyperparameter and $\mathcal{G}$ denotes the set of all triples in the knowledge graph. $(u'_i,r'_{t,i},v'_{i})$ denotes the $i^{th}$ negative triple of the ground-truth triple $(u, r_t, v)$ and $[n]$ represents the set $\{1,2,\cdots, n\}$, where $n$ is the number of negative samples for each triple.
        \subsection{Complete Common Neighbor induced Subgraph }\label{sec:saen}


\begin{table*}

    \caption{Statistics on inductive datasets when setting the neighbor hop $h$ to $2$. The values on the right of Num = 2, Num = 3, and Others denote the proportion of the corresponding type of subgraphs to the total number of extracted $2$-hop enclosing subgraphs of each dataset. Num = 2 denotes that the extracted enclosing subgraph only consists of the head entity $u$ and the tail entity $v$ with the target relation between them. Num = 3 denotes that apart from the head entity $u$ and the tail entity $v$, the extracted enclosing subgraph only remains one other entity consisting of the path from $u$ to $v$. The Incomplete$\_$Ratio is the proportion of the total nodes in the enclosing subgraph to the number of nodes in the CCN+ subgraph. The smaller the Incomplete\_Ratio value is, the more relevant rule loss is caused by the $2$-hop enclosing subgraph extraction method. }\label{tab:sta}
	\centering
		\resizebox{2.0\columnwidth}{!}{
			\begin{tabular}{c c c c c c c  c c c c  c c c c } 
				\toprule
				& &\multicolumn{4}{c}{\textbf{WN18RR}}&  \multicolumn{4}{c}{\textbf{FB15k-237}} & \multicolumn{4}{c}{\textbf{NELL-995}}\\
				\cmidrule(lr){3-6} \cmidrule(lr){7-10} \cmidrule(lr){11-14}
				& & v1    & v2    & v3    &  v4   & v1    & v2    & v3    &  v4   & v1    & v2    & v3    &  v4 \\
				\midrule
                \multirow{5}{*}{Statistics in Training Set} &
				Num=2       & 0.505 & 0.453 & 0.463 & 0.471 &  0.134 & 0.125 & 0.092 &  0.056 & 0.318 & 0.168 & 0.074 & 0.105 \\
				& Num=3            &  0.073 & 0.082 & 0.086 & 0.087 & 0.024 &  0.018 & 0.014 & 0.003 & 0.001 & 0.001 & 0.001 & 0.002 \\
				& Others           & 0.422 & 0.465 & 0.450 &  0.465 & 0.847 & 0.859 &  0.859 & 0.941 &  0.682 & 0.832 &  0.894 & 0.894 \\
    &Incomplete\_Ratio & 0.349 & 0.351 & 0.342 &  0.353 & 0.531 & 0.634 & 0.698 & 0.754 & 0.362 & 0.566 & 0.639 & 0.648 \\
                \midrule
                \multirow{5}{*}{Statistics in Tesing Set}
				& Num=2           & 0.528 &  0.505 & 0.504 &  0.510 &  0.090 &  0.049 &  0.092 &  0.062 & 0.525 &  0.221 & 0.174 & 0.174 \\ 
               & Num=3 & 0.034 & 0.029 & 0.039 & 0.042 & 0.004 & 0.002 & 0.014 & 0.004 & 0.014 & 0.003 & 0.006 & 0.001\\
               & Others & 0.456 & 0.481 & 0.477 &  0.443 &  0.906 & 0.941 & 0.897 & 0.930 & 0.469 &  0.776 & 0.822 &  0.827\\
               & Incomplete\_Ratio & 0.564 & 0.579 & 0.594 &  0.576 & 0.811 & 0.862 & 0.897 & 0.873 & 0.685 & 0.763 & 0.843 & 0.837 \\
				\bottomrule
			\end{tabular}
   
		}
	\end{table*}

	\begin{table*}[ht]
		\caption{Statistics on inductive datasets when setting the neighbor hop $h$ to $3$. The values below each version of inductive datasets denote the proportion of the extracted $3$-hop enclosing subgraphs that contain irrelevant rules to the total number of extracted $3$-hop enclosing subgraphs.}\label{Tab: 3hop sta}
		\centering
		\resizebox{2.0\columnwidth}{!}{
			\begin{tabular}{c c c c c c  c c c c  c c c c } 
				\toprule
			 	&\multicolumn{4}{c}{\textbf{WN18RR}}&  \multicolumn{4}{c}{\textbf{FB15k-237}} & \multicolumn{4}{c}{\textbf{NELL-995}}\\
				\cmidrule(lr){2-5} \cmidrule(lr){6-9} \cmidrule(lr){10-13}
				& v1    & v2    & v3    &  v4   & v1    & v2    & v3    &  v4   & v1    & v2    & v3    &  v4 \\
    \midrule
                Statistics in Training Set &0.872 & 0.756 &0.759 &0.850 & 0.979 & 0.988 & 0.989 & 0.992 & 0.969 & 0.998 & 0.999 & 0.998\\
                Statistics in Testing Set  & 0.580 & 0.501 & 0.575& 0.544& 0.615 &0.787 & 0.798 & 0.830 &0.989 & 0.918 & 0.883& 0.917\\
				\bottomrule
			\end{tabular}
		}
  
	\end{table*}

        \begin{figure*}[ht]
    \centering 
\begin{subfigure}{2\columnwidth}
\centering
  \includegraphics[width=400pt]{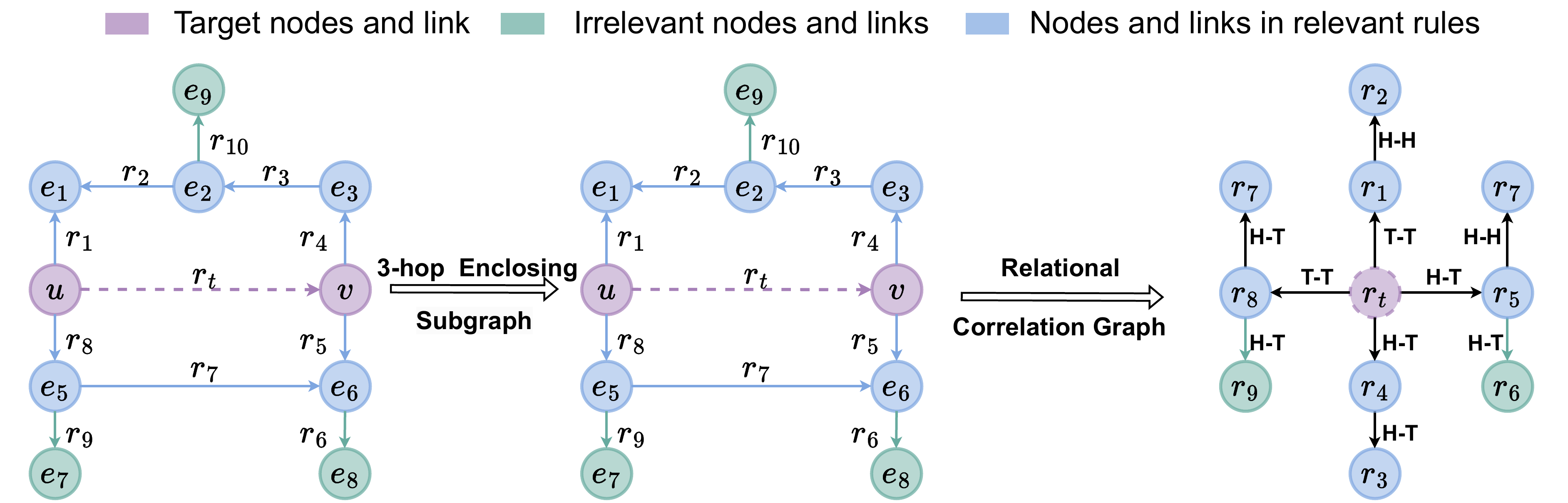}
  \label{fig:why-rf}
  \caption{The $3$-hop enclosing subgraph and corresponding relational correlation graph via adapting the enclosing subgraph {method}.}\label{fig:why-30}
\end{subfigure}
\medskip
\begin{subfigure}{2\columnwidth}
\centering
  \includegraphics[width=400pt]{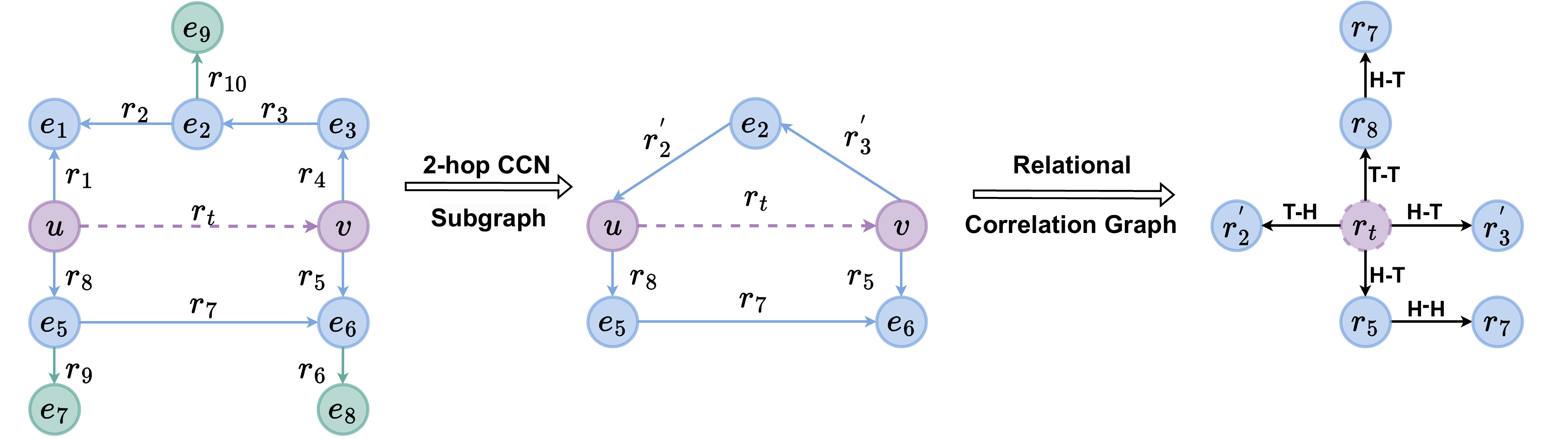}
  \label{fig:why-re}
  \caption{The $2$-hop CCN subgraph and corresponding relational correlation graph via adapting the CCN subgraph {method}.}\label{fig:why-31}
\end{subfigure}\hfil 
\medskip
\begin{subfigure}{2\columnwidth}
\centering
  \includegraphics[width=400pt]{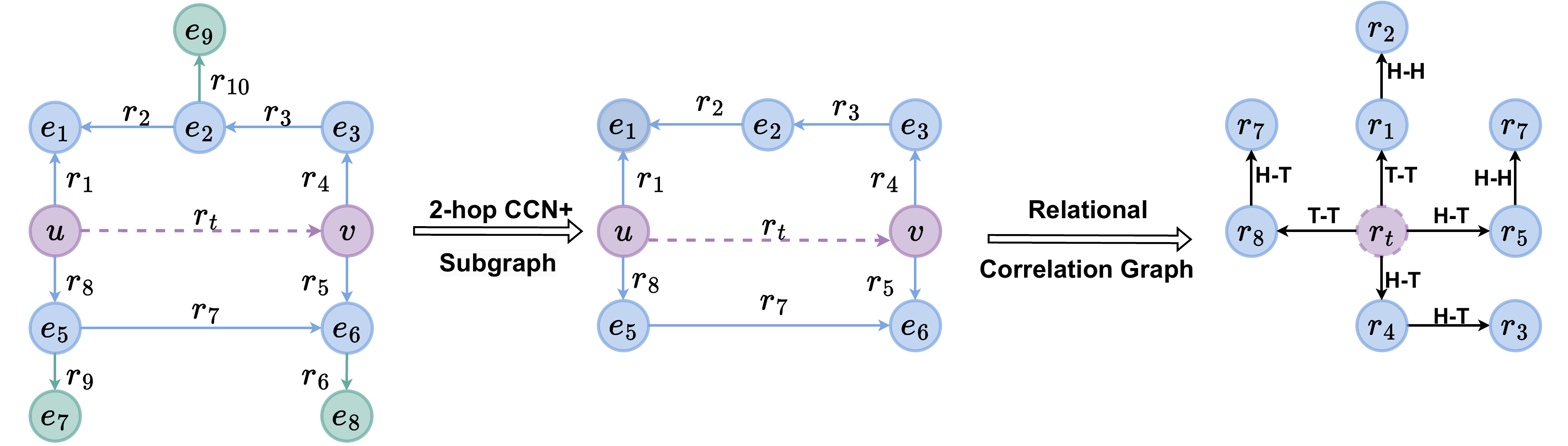}
  \label{fig:why-qa}
  \caption{The $2$-hop CCN+ subgraph and corresponding relational correlation graph via adapting the CCN+ subgraph {method}.} \label{fig:why-32}
\end{subfigure}\hfil 
\vskip -0.15in
\caption{Comparison between the $3$-hop enclosing subgraph {method}, the $2$-hop CCN subgraph {method}, and the $2$-hop CCN+ subgraph method in the same original knowledge graph.}
\label{fig:3030}
\vskip -0.1in
\end{figure*}

        \subsubsection{Motivation} \label{3.5.1}

        To further elaborate on our motivation, we introduce a rule-based learning perspective to analyze the reasoning process within the subgraph. 

         Rule-based methods aim at mining first-order logical rules based on reasoning paths from knowledge graphs. Such a rule consists of a head and a body, where a head is a single atom, i.e., a fact in the form of \textit{Relation(head entity, tail entity)}, and a body is a set of atoms. Given a head $R(y,x)$ and body $B_1\left(y, z_1\right) \wedge B_2\left(z_1, z_2\right) \wedge \cdots \wedge B_T\left(z_{T-1}, x\right)$, there is a rule $R(y,x) \leftarrow B_1\left(y, z_1\right) \wedge B_2\left(z_1, z_2\right) \wedge \cdots \wedge B_T\left(z_{T-1}, x\right)$ .
        A rule is connected if every atom shares at least one variable with another atom, and a rule is closed if each variable in the rule appears in at least two atoms \cite{neural-lp}. 

In paticular, we pay more attention to the closed and connected rules, i.e., the reasoning paths between target nodes, as closeness and connectedness prevent finding rules with irrelevant relations \cite{drum}. We define the length of a rule as the nodes that form the rule, and the relevant rules as those that are closed, connected, and have a length of no more than $2h+1$. The irrelevant rules are those that are either not closed, not connected, or have a length of more than $2h+1$, where $h$ is the hop of neighbors for the target nodes $u$ or $v$. We then define Complete Common Neighbor induced (CCN) subgraph as a subgraph that can effectively induce the relevant rules and eliminates the irrelevant rules via all common neighbors of the target nodes. To align with previous works\cite{grail}, we set the hop as $2$.

        As mentioned in Section \ref{3.3}, TACT\cite{tact} reasons on enclosing subgraphs \cite{grail}. The original intention of the enclosing subgraph is to \textit{eliminate the irrelevant rules} and \textit{preserve the relevant rules} around the target link. However, the enclosing subgraph cannot achieve both requirements, that is, the $2$-hop enclosing subgraph performs well on eliminating the irrelevant rules but not on preserving the relevant rules, while the $3$-hop enclosing subgraph performs well on preserving the relevant rules but not on eliminating the irrelevant rules.
        We propose two toy examples as illustrations and further conduct statistical analyses to reveal the prevalence of the phenomenon in Section \ref{3.5.2}.
        
        \subsubsection{Analysis of the enclosing subgraph} \label{3.5.2}

        From the perspective of rule mining, reasoning on the subgraph can also be seen as mining rules within the subgraph.
        
        The $2$-hop enclosing subgraph method suffers from properly mining the relevant rules.    
        As illustrated in Figure \ref{fig:en}, when extracting the $2$-hop enclosing subgraph of the original knowledge graph, the $2$-hop enclosing subgraph will eliminate all other entities and their corresponding relations. 
        To further elaborate on the prevalence of the observation, we conduct a statistical analysis by setting the neighborhood hop as $2$ on the inductive datasets in Table \ref{tab:sta}.  
From Table \ref{tab:sta}, we can see that the enclosing subgraph extraction method incurs a significant loss of the relevant rules. Specifically, for WN18RR\cite{wn18rr}, nearly half of the extracted subgraphs only contain the target nodes $u$ and $v$.   The Incomplete$\_$Ratio also supports the statement. Varied degrees of relevant rule loss are also discernible in FB15K-237\cite{conve} and NELL-995\cite{xiong2017deeppath} datasets. This may cause a significant loss of relevant nodes and relations, which hinders the model's performance.

The $3$-hop enclosing subgraph suffers from eliminating the irrelevant rules. As illustrated in Figure \ref{fig:why-30}, when extracting the $3$-hop enclosing subgraph of the original knowledge graph, the $3$-hop enclosing subgraph cannot eliminate the irrelevant nodes $e_7$, $e_8$, $e_9$ and their corresponding relations. From Table \ref{Tab: 3hop sta}, we can see that the majority of the $3$-hop enclosing subgraphs contain irrelevant rules. This may let the model overfit to the irrelevant nodes and relations,  which hinders the model performance as well.

Thus, to address these problems, we propose two Complete Common Neighbor induced subgraph extraction methods, namely the CCN subgraph {method} and the CCN+ subgraph {method}.

\subsubsection{The CCN subgraph {method} }\label{da:pip}
Hereby, instead of pruning the isolated common neighbors after calculating $\mathcal{N}_{k}(u) \cap \mathcal{N}_{k}(v)$ like enclosing subgraphs in Section \ref{sec:subex}, we label their relations with an additional distance coordinate, which represents the distance of the isolated nodes to the target nodes $u$ and $v$ in the original knowledge graph. The relation between isolated nodes and the target nodes $u$ or $v$ is the same as the isolated nodes to their adjacent nodes in the original knowledge graph, while we label them with additional distance coordinates. For instance, as shown in Figure \ref{fig:saen}, the $2$-hop CCN subgraph preserves the isolated nodes $e_2$ and replace the relations $r_2$ and $r_3$ with the equivalent relations to the target node, that is, $r_2^{'}=r_2 \oplus (d(e_2,u),d(e_2,v))$ and $r_3^{'}=r_3 \oplus (d(e_2,u),d(e_2,v))$, where $\oplus$ denotes the concatenation of vectors and $d(i,u)$ denotes the shortest distance between nodes $i$ and $u$ in the original knowledge graph without counting any path through $v$. The additional distance coordinates indicate that the isolated nodes are on one relevant rule from $u$ to $v$, whereas the enclosing subgraph method cannot preserve the isolated nodes.

\begin{algorithm}[ht]
\caption{Pseudo code for extracting CCN subgraph}.
\label{alg:cni+}
    \begin{algorithmic}[1]
    \STATE \textbf{Input} the target prediction link $(u,r_t,v)$ and the knowledge graph $\mathcal{G}$.
    \STATE Define the $k$-hop neighbor of node $u$ as $\mathcal{N}_{k}(u)$ and the distance in the original knowledge graph between node $i$ and node $u$ as $d(i,u)$. $\oplus$ denotes the concatenation of two vectors.
    \STATE Compute the intersection of $k$-hop neighbor of the target nodes $u$ and $v$ to get the common neighbor set, $S=\mathcal{N}_{k}(u) \cap \mathcal{N}_{k}(v)$.  
     \FOR {Isolated Node $i$ in $S$} 
        \STATE 
        Do the \textbf{Label Procedure} for nodes $i, u$ and $i, v$ respectively to get the labeled relevant nodes and relations.
    \ENDFOR
    
  \STATE \textbf{Label Procedure} 
  \STATE Compute $d(i,j)$ for node $i$ and the target node $j$.
  \STATE Label the isolated common neighbor $i$ with $d(i,j)$.
  \STATE Label the relation between nodes $i$ and $j$ with $r^{'}= r \oplus d(i,j)$, where $r$ is the same as the isolated common neighbor node $i$ to its adjacent $j$'s neighboring node in the original knowledge graph.
\STATE Update the Labeled node $i$ in the intersection set $S$.
\STATE \textbf{End Label Procedure}
    
    \STATE Assigning the results $S$ of S to  CCN$\_$Subgraph
 
    \RETURN CCN$\_$Subgraph
    \end{algorithmic}
\end{algorithm}

As shown in Figure \ref{fig:saen},
compared with the $2$-hop enclosing subgraph, after adopting the $2$-hop CCN subgraph extraction method, isolated common neighbor $e_2$ will be preserved as an indicator for the existence of the rule from head entity $u$ to the tail entity $v$ through $e_2$. It preserves complete common neighbors and more relevant relations from the subgraph efficiently. And compared with the $3$-hop enclosing subgraph, as shown in Figure \ref{fig:why-31}, the $2$-hop CCN subgraph effectively eliminates the irrelevant nodes $e_7$, $e_8$, $e_9$ and their corresponding relations. We summarize the procedure of CCN subgraph algorithm in Algorithm \ref{alg:cni+}. And
we call TACO that reasons on CCN subgraph as TACO$_{CCN}$, abbreviated as TACO.

\begin{algorithm}[ht]
\caption{Pseudo code for extracting CCN+ subgraph}.
\label{alg:enclo}
    \begin{algorithmic}[1]
    \STATE \textbf{Input} the target prediction link $(u,r_t,v)$ and the knowledge graph $\mathcal{G}$.
    \STATE Define $\mathcal{N}_i(u)$ as $i$ hop neighbor of the node $u$, $\mathcal{B}_i(u)$ as the $i^{th}$ hop neighbor of $u$ and $\mathcal{N}_0(u)=u$ as the zero hop neighbor of $u$.
     \FOR {$i=1,2,...,k$} 
        \STATE $\mathcal{B}_i(u)$=$\mathcal{N}_i(u)\setminus \mathcal{N}_{i-1}(u)$,
        $\mathcal{B}_i(v)$=$\mathcal{N}_i(v)\setminus \mathcal{N}_{i-1}(v)$
    \ENDFOR
    \STATE Do the \textbf{Distilled Procedure} for node $u$ and $v$ to get distilled relevant neighbors $\mathcal{N}_i(u)$ and $\mathcal{N}_i(u)$, respectively.
  \STATE \textbf{Distilled Procedure} \FOR {$i = k-1,..,1$}
    \FOR {node in $\mathcal{N}_i(u)$}
    \IF {$\mathcal{N}_1(node)\cap \mathcal{B}_{i+1}(u)=\phi$}
   \STATE $\mathcal{N}_i(u) = \mathcal{N}_i(u) \setminus node$
    \ENDIF
    \ENDFOR
    \ENDFOR
    \STATE \textbf{End Distilled Procedure}
    \STATE Compute the $k$ hop common neighbor sets,  $Int\_k = \mathcal{N}_k(u)\bigcap\mathcal{N}_k(v)$
    \STATE Compute the $i^{th}$ hop $ (i=0,1\cdots,k)$ distilled union neighbors, $Distilled\_k = \bigcup_{i=0}^{k-1}\mathcal{N}_i(u)\bigcup_{i=0}^{k-1}\mathcal{N}_i(v)$
    
    \STATE Compute the CCN+\_Subgraph = $Int\_k \bigcup Distilled\_k$
 
    \RETURN CCN+$\_$Subgraph
    \end{algorithmic}
\end{algorithm}

\subsubsection{The CCN+ subgraph {method} }\label{ca:pip}
The CCN subgraph preserves the proportion of relevant rules in an efficient approach, while it still exists the relevant rule loss issue. As illustrated in Figure \ref{fig:saen} and Figure \ref{fig:why-31}, the entities $e_1$ and $e_3$ have been eliminated when executing this method. Thus, we further propose the CCN+ subgraph method to fully mine and preserve the relevant rules from the target head entity $u$ and the target tail entity $v$. Specifically, we notice that the nodes in the $i^{th}$ hop of the relevant rules are linked by the nodes in the $(i+1)^{th}$ hop of the relevant rules. Based on this observation, we propose the CCN+ {method} to fully preserve the relevant rules.

The CCN+ method can effectively distinguish the relevant and irrelevant nodes. As Figures \ref{fig:03} and \ref{fig:why-32} illustrate, The CNN+ subgraph can properly preserve the relevant rules and eliminate the irrelevant rules. 
We summarize the procedure of the CCN+ subgraph algorithm in Algorithm \ref{alg:enclo}. And we call TACO that reasons on  CCN+ subgraph as TACO$_{CCN+}$, abbreviated as TACO$_+$.

\begin{table}[ht]
		\caption{Statistics of inductive benchmarks. We use \#E and \#R, and \#TR to denote the number of entities, relations, and triples, respectively.}
		\centering
		\resizebox{1.0\columnwidth}{!}{
			\begin{tabular}{l l *{9}{c}}
				\toprule
				& &\multicolumn{3}{c}{\textbf{WN18RR}}&  \multicolumn{3}{c}{\textbf{FB15k-237}} & \multicolumn{3}{c}{\textbf{NELL-995}}\\
				\cmidrule(lr){3-5}\cmidrule(lr){6-8}\cmidrule(lr){9-11}
				& &\#R &\#E &\#TR     &\#R &\#E &\#TR    &\#R &\#E &\#TR   \\
				\midrule
				\multirow{2}{*}{v1} & train & 9 & 2746 & 6678 &   183 & 2000 & 5226 &  14 & 10915 & 5540\\
				& test  & 9 & 922 & 1991 &   146 & 1500 & 2404 &  14 & 225 & 1034\\
				\midrule
				\multirow{2}{*}{v2} & train & 10 & 6954 & 18968 &  203 & 3000 & 12085 &  88 & 2564 & 10109\\
				& test  & 10 & 2923 & 4863 &   176 & 2000 & 5092 &  79 & 4937 & 5521\\
				\midrule
				\multirow{2}{*}{v3} & train & 11 & 12078 & 32150 & 218 & 4000 & 22394 &  142 & 4647 & 20117\\
				& test  & 11 & 5084 & 7470 &   187 & 3000 & 9137 &  122 & 4921 & 9668\\
				\midrule
				\multirow{2}{*}{v4} & train & 9 & 3861 & 9842 &   222 & 5000 & 33916 &  77 & 2092 & 9289\\
				& test  & 9 & 7208 & 15157 &   204 & 3500 & 14554 &  61 & 3294 & 8520\\
				\bottomrule
			\end{tabular}
		}
		\label{Tab:ind-data}
	\end{table}

	\begin{table*}[ht]
		\caption{AUC-PR results on inductive benchmark datasets. The results of Neural-LP, DURM, RuleN, GraIL, and CoMPLIE are taken from  the CoMPLIE \cite{compile} and the results of SNRI are taken from SNRI \cite{SNRI}.} \label{Tab: AUC-PR}
		\centering
		\resizebox{2.0\columnwidth}{!}{
			\begin{tabular}{l c c c c c  c c c c  c c c c } 
				\toprule
				&\multicolumn{4}{c}{\textbf{WN18RR}}&  \multicolumn{4}{c}{\textbf{FB15k-237}} & \multicolumn{4}{c}{\textbf{NELL-995}}\\
				\cmidrule(lr){2-5} \cmidrule(lr){6-9} \cmidrule(lr){10-13}
				& v1    & v2    & v3    &  v4   & v1    & v2    & v3    &  v4   & v1    & v2    & v3    &  v4 \\
				\midrule
				Neural LP       & 86.02 & 83.78 & 62.90 & 82.06 & 69.64 & 76.55 & 73.95 & 75.74 & 64.66 & 83.61 & 87.58 & 85.69 \\
				DRUM            & 86.02 & 84.05 & 63.20 & 82.06 & 69.71 & 76.44 & 74.03 & 76.20 & 59.86 & 83.99 & 87.71 & 85.94 \\
				RuleN           & 90.26 & 89.01 & 76.46 & 85.75 & 75.24 & 88.70 & 91.24 & 91.79 & 84.99 & 88.40 & 87.20 & 80.52 \\
				GraIL           & 94.32 & 94.18 & 85.80 & 92.72 & 84.69 & 90.57 & 91.68 & 94.46 & 86.05 & 92.62 & 93.34 & 87.50 \\ 
                CoMPILE & 98.23 & {99.56} & 93.60 & {99.80} & 85.50 & 91.86 & 93.12 & 94.90 & 80.16 & 95.88 & 96.08 & 85.48 \\
                NBFNet & 98.39 & 98.96 & 94.37 & 99.12 & 93.06 & 96.88 & 97.05 & 97.83 & 98.30 & 98.22 & 97.89 & 98.22\\
                SNRI &99.10  &99.92  &94.90 & 99.61 &86.69 &91.77  &91.22& 93.37 &-&-&-&-\\
                \midrule
				TACT-base    & {98.11} & 97.11 & 88.34 & {97.25} & 87.36 & {94.31} & {97.42} & 98.09 & 94.00 & 94.44 & 93.98 & 94.93 \\
				TACT  & 96.15 & {97.95} & {90.58} & 96.15 & {88.73} &  94.20 & 97.10 & {98.30} & {94.87}& {96.58} & {95.70} & {96.12} \\
				\midrule
                TACO-base & 98.90 & 97.94 & 91.23 & 97.85 & 92.12 & 96.87 & 98.08 & 98.34 & 99.60 & {99.27} & 99.07 & 98.54 \\ 
                TACO & {99.27} & {98.41} & {93.90} & {99.27} & \textbf{93.97} & \textbf{97.40} & {98.83} & {99.39} & {99.69} & {99.17} & {99.30} & {99.07}\\
                TACO$_+$-base &\textbf{99.73} &\textbf{99.94} &95.26 &97.92 &91.73 &96.67 &98.80 &98.10 &99.30 &99.06 &99.09 &98.64\\

                TACO$_+$ & 99.14 & {99.77} & \textbf{97.75} & \textbf{99.94} & 92.54 & 96.45 & \textbf{99.66} & \textbf{99.43} & \textbf{99.95} & \textbf{99.92} & \textbf{99.98} & \textbf{99.56} \\
                \bottomrule
                
			\end{tabular}

		}
  
	\end{table*}

\section{Experiments and Analysis}
	This section is organized as follows. First, we introduce the experimental configurations, including inductive datasets in Section \ref{sec:dataset}, implementation details of the TACO in Section \ref{sec: training protocol}, and the baseline model, TACO-base in Section \ref{sec:basemodel}. Second, we demonstrate the effectiveness of our proposed approach TACO on several inductive benchmark datasets in Section \ref{sec:ilp}. Finally, we conduct the ablation studies, case studies, and further experiments in Sections \ref{sec:ablation} and \ref{sec:further}.

     \subsection{Datasets}\label{sec:dataset}
     We use the benchmark datasets for link prediction proposed in GraIL \cite{grail}, which are derived from WN18RR \cite{wn18rr}, FB15k-237 \cite{conve}, and NELL-995 \cite{xiong2017deeppath}. For inductive link prediction, the training set and the testing set should have no overlapping entities. Each knowledge graph of WN18RR, FB15k-237, and NELL-995 induces four versions of inductive datasets with increasing sizes. Details of the datasets are summarized in Table \ref{Tab:ind-data}.


 \subsection{Training protocol} \label{sec: training protocol}
 We randomly sample $2$-hop CCN subgraphs for each triple when training and testing. We apply a two-layer GCN to calculate the embeddings of subgraphs. The embedding dimension of node entities and relations is set to 32. The margins in the loss functions are set to 8, 16, and 10 for WN18RR, FB15k-237, and NELL-995, respectively. Implementation details  are summarized in Appendix \ref{imp_details}.
 
 \subsection{The baseline model}\label{sec:basemodel}

 To evaluate the effectiveness of the proposed relational correlation module, we propose a baseline called TACO-base, which scores a triple $(u, r_t, v)$ only relying on the output of the RCN module. That is, the score function is 
	\begin{align*}
		f_{\text{base}}(u, r_t, v) = \textbf{r}_t^{F} \textbf{W}_{base}
	\end{align*}
	where $\textbf{W}_{\text{base}}\in \mathbb{R}^{d\times 1}$ is weight parameters.

	\subsection{Inductive Link Prediction} \label{sec:ilp}
	We evaluate the models on both classification and ranking metrics. For both metrics, we compare our method to several state-of-the-art methods, including Neural LP \cite{neural-lp}, DRUM \cite{drum}, RuleN\cite{rulen}, GraIL \cite{grail}, CoMPILE \cite{compile}, and SNRI\cite{SNRI}.

\subsubsection{Classification metric}

We use the area under the precision-recall curve (AUC-PR) as the classification metric following GraIL \cite{grail}. We replace the head or the tail entity of every test triple with a random entity to sample the corresponding negative triples. Then we score the positive triples with an equal number of negative triples to calculate AUC-PR following GraIL \cite{grail}. To make the results more reliable, we run each experiment five times with different random seeds and report the mean results.

        From the AUC-PR results in Table \ref{Tab: AUC-PR}, we observe that on the twelve versions of the three datasets, our model TACO-base and TACO have reached the optimum in all twelve AUC-PR values. Specifically, for both CCN and CCN+ methods, TACO outperforms rule-based baselines, including Neural LP, DRUM, and RuleN by a significant margin. Compared with the subgraph-based method GraIL, CoMPILE, TACT, and SNRI, the best performance TACO can achieve average AUC-PR improvements of 7.58\%, 7.27\%, 9.98\%; 1.54\%, 6.27\%, 10.47\% ; 4.13\%, 3.03\%, 4.04\% and 0.98\%, 6.85\% , ``-'' on three datasets respectively, which demonstrates the superiority of TACO. And it also outperforms the path-based method NBFNet with average AUC-PR improvements of 1.63\%; 1.41\%;1.70\% on all three datasets respectively. 
        
        As TACO-base totally relies on the relational correlation module to perform link prediction, the results demonstrate the effectiveness of our proposed model for inductive link prediction. It also means that the RCN module can capture the relation correlations effectively. TACO further improves the performance of TACO-base in the majority of datasets and achieves further improvement against GraIL, CoMPILE, and NBFNet on the benchmark datasets. The experiments demonstrate the effectiveness of modeling edge-level topology-aware correlations between relations in TACO for the inductive link prediction task.

	\begin{table*}[ht]
		\caption{Hits@10 results on inductive benchmarks datasets. The results of Neural LP, DURM, RuleN, GraIL, and CoMPILE are taken from CoMPILE\cite{compile}. The results of SNRI are taken from SNRI\cite{SNRI}. }\label{Tab: Ranking}
		\centering
		\resizebox{2.0\columnwidth}{!}{
			\begin{tabular}{l c c c c c  c c c c  c c c c } 
				\toprule
				&\multicolumn{4}{c}{\textbf{WN18RR}}&  \multicolumn{4}{c}{\textbf{FB15k-237}} & \multicolumn{4}{c}{\textbf{NELL-995}}\\
				\cmidrule(lr){2-5} \cmidrule(lr){6-9} \cmidrule(lr){10-13}
				& v1    & v2    & v3    &  v4   & v1    & v2    & v3    &  v4   & v1    & v2    & v3    &  v4 \\
				\midrule
				Neural LP       & 74.37 & 68.93 & 46.18 & 67.13 & 52.92 & 58.94 & 52.90 & 55.88 & 40.78 & 78.73 & 82.71 & {80.58} \\
				DRUM            & 74.73 & 68.93 & 46.18 & 67.13 & 52.92 & 58.73 & 52.90 & 55.88 & 19.42 & 78.55 & 82.71 & {80.58} \\
				RuleN           & 80.85 & 78.23 & 53.39 & 71.59 & 49.76 & 77.82 & 87.69 & 85.60 & 53.50 & 81.75 & 77.26 & 61.35 \\
				GraIL           & 82.45 & 78.68 & 58.43 & 73.41 & 64.15 & 81.80 & 82.83 & 89.29 & 59.50 & 93.25 & 91.41 & 73.19 \\ 
                CoMPILE         & 83.60 & 79.82 & 60.69 & 75.49 & 67.66 & 82.98 & 84.67 & 87.44 & 58.38 & {93.87} & 92.77 & 75.19\\
                SNRI    & 87.23 & 83.10 & 67.31 & 83.32 & 71.79 & 86.50 & 89.59 & 89.39 & - & - & - & -\\
                \midrule
                TACT-base & 81.38 & 77.64 & 58.76 & 73.47 & 64.61 & 82.72 & 86.72 & 89.71 & 58.50 & {92.16} & 91.04 & 71.33
                \\
                TACT & 81.69 & 80.06 & 62.32 & 74.69 & {65.48} & {84.25} & 85.62 & 88.04 & {58.03} & 91.17 & {90.72} & {73.42}\\

				\midrule
				TACO-base    & {90.08} & \textbf{92.08} & \textbf{89.67} & {90.69} & 78.05 & {87.17} & {88.32} & 87.15 & 60.50 & {94.01} & {93.30} & {82.79} \\
				TACO  & 91.09 & {91.40} & {85.45} & 88.59 & \textbf{82.01} &  86.19 & 86.53 & {89.96} & {58.50} & {93.30} & \textbf{94.48} & \textbf{84.33}\\
                TACO$_+$-base & {93.22}  & 88.73 & {79.84} & {90.71} & {79.93} & \textbf{88.52} & \textbf{89.62} & 90.84 & \textbf{62.00} & 93.71 & {93.11} & {79.09}\\ 
				TACO$_+$  & \textbf{95.12} & {90.46} & {82.02} & \textbf{91.37} & {75.93} &  86.83 & {85.59} & \textbf{91.17} & {60.81} & \textbf{94.11} & {94.28} & {77.19} \\
                \bottomrule
			\end{tabular}
		}
	\end{table*}

 \subsubsection{Ranking metric}

 	We further evaluate the model for inductive link prediction to verify the effectiveness of modeling relational correlations in TACO. We rank each test triple among 50 other randomly sampled negative triples. Specifically, for a given relation prediction $(u, r_t, ?)$ or $(?, r_t, v)$  in the testing set, we rank the ground-truth triples $(u, r_t, v)$ against all other candidates negative triples. Following the standard procedure in prior work \cite{transe}, we use the filtered setting, which does not take any existing valid triples into account at ranking. We choose Hits at N (H@N) as the evaluation metric. Following GraIL\cite{grail}, to make the results more reliable, we run each experiment five times with different random seeds and report the mean results.

	Table \ref{Tab: Ranking} shows the results H@10 on WN18RR, FB15k-237, and NELL-995 from version 1 to version 4. As we can see, TACO significantly outperforms rule learning based methods Neural-LP, DRUM, and RuleN \cite{neural-lp, drum, rulen}; subgraph based methods GraIL, TACT, CoMPILE, and SNRI \cite{grail, compile,tact,SNRI} in all datasets by a significant margin. 
 In this scenario, TACO achieves a maximum of 27.35\% (on WN18RR v3) and 15.88\% (on WN18RR v4).
 Compared with the subgraph-based method GraIL, CoMPILE, TACT, and SNRI, the best performance TACO can achieve average AUC-PR improvements of 18.82\%, 8.31\%, 4.37\%; 17.16\%, 7.15\%, 3.65\% ; 17.38\%, 6.29\%, 3.62\% and 11.82\%, 3.51\% , ``-'' on three datasets respectively, which demonstrates the superiority of TACO.
    
    As the aforementioned Table \ref{tab:sta}, the Incomplete$\_$Ratio serves as an indicator of the relevant rule loss degree incurred by the enclosing subgraph method. The smaller the Incomplete$\_$Ratio value is, the more relevant rule loss is caused by the enclosing subgraph method. FB15k-237 exhibits the highest Incomplete$\_$Ratio values in both training and testing sets, thus the proposed CCN subgraph methods exhibit relatively modest improvements compared with the datasets WN18RR and NELL995. 
    
    The experiments also show that GraIL has difficulty in modeling relational semantics, especially when the number of relation types is large. In contrast, TACO can model the complex patterns of relations by exploiting correlations between relations in knowledge graphs. TACO-base also significantly outperforms the existing subgraph-based state-of-the-art methods.
    Notably, in ranking tasks, a large number of negative samples can bring about significant performance improvements. Subgraph-based methods use a negative sampling rate of 1, while the path-based method NBFnet uses a negative sampling rate of 32. Therefore, we do not compare the path-based method in the ranking task.
    Also, the running time of the path-based method is longer than the subgraph-based methods due to the large negative sampling rate. We report the comparison in Section  \ref{sec:running}.

	\begin{table*}[ht]
		\caption{Ablation results of triples classification on the inductive benchmark datasets extracted from WN18RR, FB15k-237, and NELL-995. ``TACO$_+$  w/o RA" represent the baseline that omits the relation aggregation in TACO$_+$. ``TACO$_+$  w/o RC" represent the baseline that performs relation aggregation without modeling correlations between relations in TACO$_+$.}\label{Tab: ablation-1}
		\centering
		\resizebox{2.0\columnwidth}{!}{
			\begin{tabular}{l c c c c c  c c c c  c c c c } 
				\toprule
				&\multicolumn{4}{c}{\textbf{WN18RR}}&  \multicolumn{4}{c}{\textbf{FB15k-237}} & \multicolumn{4}{c}{\textbf{NELL-995}}\\
				\cmidrule(lr){2-5} \cmidrule(lr){6-9} \cmidrule(lr){10-13}
				& v1    & v2    & v3    &  v4   & v1    & v2    & v3    &  v4   & v1    & v2    & v3    &  v4 \\
                TACO$_+$  w/o RA & 97.75 & 97.33 & 91.95& 97.64& 92.06 & 95.83 & 97.05 & 97.42 & 98.06 & 97.33 & 96.46 & 98.26\\
                TACO$_+$  w/o RC &98.82 & 93.17 & 90.29 & 97.85 & 92.39 & 95.97 & 96.69 & 99.22 & 98.73 & 97.32 & 98.81 & 98.39\\
                \midrule
                TACO$_+$ & \textbf{99.14} & \textbf{99.77} & \textbf{97.75} & \textbf{99.94} & \textbf{92.54} & \textbf{96.45} & \textbf{99.66} & \textbf{99.43} & \textbf{99.95} & \textbf{99.92} & \textbf{99.98} & \textbf{99.56} \\
				\bottomrule
			\end{tabular}
		}
	\end{table*}

	\begin{table*}[ht]
		\caption{Ablation study for investigating effect of each part of input embeddings in the scoring network of TACO$_+$. The symbols $\textbf{n}$, $\textbf{g}$, and $\textbf{r}$ represent the node embedding, the graph embedding, and the final relation embedding, respectively.}\label{Tab: ablation-2}
		\centering
		\resizebox{2.0\columnwidth}{!}{
			\begin{tabular}{c c c  c c c c c  c c c c  c c c c } 
				\toprule
				&&&\multicolumn{4}{c}{\textbf{WN18RR}}&  \multicolumn{4}{c}{\textbf{FB15k-237}} & \multicolumn{4}{c}{\textbf{NELL-995}}\\
				\cmidrule(lr){4-7} \cmidrule(lr){8-11} \cmidrule(lr){12-15}
				\textbf{n} & \textbf{g} & \textbf{r} & v1 & v2 & v3 & v4 & v1 & v2 & v3 & v4 & v1 & v2 & v3 & v4\\
				\midrule
				&&\checkmark           & 98.90 & 97.94 & 95.03 & 97.85 & 92.12 & 96.87 & 98.08 & 98.34 & 99.60 & {99.27} & 99.07 & 98.54 \\ 
				&\checkmark&\checkmark & {96.36} & \textbf{99.86} & {95.10}    & 97.10  & 92.62 & 97.35 & {99.10} & {98.72} & 99.01  & \textbf{99.96} & 98.07 & {99.14} \\
				\checkmark&&\checkmark   & {97.32} & 96.55 & 96.31    & 98.01 & \textbf{93.30}&	\textbf{97.53} & {98.97} & 98.75 & {99.81}  & {99.59} & {99.32} & {99.12} \\ \midrule
				\checkmark&\checkmark&\checkmark& \textbf{99.14} & {99.77} & \textbf{97.75} & \textbf{99.94} & {92.54} & {96.45} & \textbf{99.66} & \textbf{99.43} & \textbf{99.95} & {99.92} & \textbf{99.98} & \textbf{99.56}\\
				\bottomrule
			\end{tabular}
		}
	\end{table*}

	\subsection{Ablation Studies}\label{sec:ablation}
 
	\begin{table*}[ht]
		\caption{Ablation study for investigating the effect of each part of different relation correction patterns in TACO$_+$. The symbol ``\checkmark'' denotes the types of  relation correlation patterns that are taken into account by the RCN module. }\label{Tab: ablation-NUM}
		\centering
		\resizebox{2.0\columnwidth}{!}{
			\begin{tabular}{c c c c c c  c c c c c  c c c c  c c c c } 
				\toprule
				&&&&&&\multicolumn{4}{c}{\textbf{WN18RR}}&  \multicolumn{4}{c}{\textbf{FB15k-237}} & \multicolumn{4}{c}{\textbf{NELL-995}}\\
				\cmidrule(lr){7-10} \cmidrule(lr){11-14} \cmidrule(lr){15-18}
				\textbf{H-T} & \textbf{T-T} & \textbf{H-H} &\textbf{T-H}&\textbf{PARA} &\textbf{LOOP} & v1 & v2 & v3 & v4 & v1 & v2 & v3 & v4 & v1 & v2 & v3 & v4\\
				\midrule
				\checkmark&&&&& &{98.92} & {97.62} & {96.33} & {96.18} & {88.22} & {86.74} & {96.63} & \textbf{95.13} & {97.67} & {93.65} & {96.11} & {95.52}\\
				\checkmark&\checkmark&&&& &{98.68} & {97.98} & {97.33} & {97.79} & {87.23} & {89.81} & {97.96} & {95.16} & {97.99} & {93.42} & {96.37} & {98.69}\\
				\checkmark&\checkmark&\checkmark&&& &{98.81} & {98.08} & \textbf{99.00} & {98.55} & {91.10} & {89.80} & {98.64} & {95.89} & \textbf{99.98} & {94.01} & {96.43} & {98.45}\\
				\checkmark&\checkmark&\checkmark&\checkmark&&   & {98.26} & 99.25 & 98.98    & 99.35 & 91.61&	{91.28} & {98.62} & 98.66 & {99.13}  & {94.69} & {99.57} & {98.83} \\ 
				\checkmark&\checkmark&\checkmark&\checkmark&\checkmark& & \textbf{99.97} & {99.27} & {97.24}    & 98.86  & 91.01 & 94.15 & {99.17} & {98.86} & 99.70  & \textbf{99.94} & 99.80 & {98.71} \\
    
				\checkmark&\checkmark&\checkmark&\checkmark&\checkmark&\checkmark           & 99.14 & \textbf{99.77} & 97.75 & \textbf{99.94} & \textbf{92.54} & \textbf{96.45} & \textbf{99.66} & \textbf{99.43} & {99.95} & {99.92} & \textbf{99.98}& \textbf{99.56} \\

				\bottomrule
			\end{tabular}
		}
	\end{table*}

	In this part, we conduct the ablation studies on the relation aggregation and the topology-aware correlations between relations in Section \ref{ab: w/o ra}. To further investigate the effect of each part of input embeddings in the scoring network and different relation correlation patterns taken into account by the RCN module in Section \ref{ab: input emb} and Section \ref{ab: diff rc}.
 \subsubsection{Ablation on relation aggregation and  coefficients} \label{ab: w/o ra}
 	In our proposed method, we aggregate the relation embedding $\textbf{r}_t$ and neighborhood embedding $\textbf{r}_t^N$ to get the final relation embedding $\textbf{r}_t^F$. We omit the aggregation of neighborhood embedding, that is, we let the output of the relational correlation module be $\textbf{r}_t$. We reason on the CCN+ subgraph as an example and call this method ``TACO$_+$ w/o RA" for short.
  
  	Modeling topology-aware correlations between relations is one of our main contributions. We design a baseline that performs relation aggregation without modeling correlations between relations. That is, the baseline reformulates the equation \eqref{rel-agg} as 
	\begin{align*}
		\textbf{r}_t^{N} = \frac{1}{|\mathcal{N}(r_t)|}\sum_{i\in \mathcal{N}(r_t)} \textbf{r}_i
	\end{align*}
	where $\mathcal{N}(r_t)$ represents the set of neighborhood relations of $r_t$. We reason on the CCN+ subgraph as an example and call this baseline ``TACO$_+$ w/o RC" for short.

 	Table \ref{Tab: ablation-1} shows the results on three benchmark datasets. The experiments demonstrate the effectiveness of modeling topology-aware correlations between relations in TACO$_+$. As correlations between relations are common in knowledge graphs, the relation aggregation in ``TACO$_+$ w/o RC" can take advantage of neighborhood relations, which is helpful for inductive link prediction. Our proposed method further distinguishes the correlation patterns and correlation coefficients between relations, which makes the learned embeddings of relations more expressive for inductive link prediction. As we can see, TACO$_+$ significantly and consistently outperforms ``TACO$_+$ w/o RA" and ``TACO$_+$ w/o RC" on all the inductive datasets. 

\begin{table}[ht]
		\caption{Some relations and their top 3 relevant relations. The relations are taken from WN18RR and NELL-995. We use CP to represent the correlation pattern and use CC to the represent correlation coefficient.}
		\resizebox{1.0\columnwidth}{!}{
			\begin{tabular}{l c c c}
				\toprule
				Target relation  &  Most relevant relations & CP & CC \\
				\midrule
				
				&\textit{\_has\_part}     & PARA                  &0.68 \\
				\textit{\_member\_meronym} &\textit{\_similar\_to} & H-H          &0.39 \\
				&\textit{\_synset\_domain\_topic\_of} & T-H     &0.31 \\
				\midrule  
				&\textit{\_similar\_to}   & LOOP                 &0.40 \\
				\textit{\_similar\_to} & \textit{\_member\_meronym} & H-H        &0.39 \\
				&\textit{\_instance\_hypernym} & T-T           &0.35 \\
				\midrule
				& \textit{television\_station\_affiliated\_with} & H-H &0.52 \\
				\textit{head\_quartered\_in} & \textit{head\_quartered\_in} & PARA &0.33 \\
				& \textit{acquired}       & T-H                 &0.30 \\
				\midrule
				&\textit{\_hypernym}     & T-T       &0.71 \\
				\textit{\_member\_of\_domain\_usage} &\textit{\_similar\_to} & H-H  &0.50 \\
				&\textit{\_derivationally\_related\_form} & H-T     &0.42 \\

				\bottomrule
    
		\end{tabular}}
		\label{tab:ex2}
	\end{table}

	\begin{table*}[ht]
		\caption{Comparison between the $2$-hop TACO, $2$-hop TACO$_+$, and $3$-hop enclosing subgrah TACT in the triple classification task.}\label{Tab: comparison}
		\centering
		\resizebox{2.0\columnwidth}{!}{
			\begin{tabular}{l c c c c c  c c c c  c c c c } 
				\toprule
				&\multicolumn{4}{c}{\textbf{WN18RR}}&  \multicolumn{4}{c}{\textbf{FB15k-237}} & \multicolumn{4}{c}{\textbf{NELL-995}}\\
				\cmidrule(lr){2-5} \cmidrule(lr){6-9} \cmidrule(lr){10-13}
				& v1    & v2    & v3    &  v4   & v1    & v2    & v3    &  v4   & v1    & v2    & v3    &  v4 \\
                $2$-hop TACO  &\textbf{99.27} & 98.41 & 93.90 & 99.27 & \textbf{93.97} & \textbf{97.40} & 98.83 & 99.39 & 99.69 & 99.17 & 99.30 & 99.07\\
                $2$-hop TACO$_+$  & 99.14 & \textbf{99.77} & \textbf{97.75}& \textbf{99.94}& 92.54 & 96.45 & \textbf{99.66} & \textbf{99.43} & \textbf{99.95} & \textbf{99.92} & \textbf{99.98} & \textbf{99.56}\\
                $3$-hop TACT & {97.79} & {96.43} & {88.15} & {81.57} & {88.34} & {94.42} & {97.16} & {98.16} & {93.95} & {95.97} & {93.83} & {94.76} \\
				\bottomrule
			\end{tabular}
		}
	\end{table*}

	\begin{table}[ht]
		\caption{The hits@10 of the frequency-based method, TACT, and the proposed TACO$_+$.}
		\begin{center}
			\resizebox{1.0\columnwidth}{!}{
				\begin{tabular}{c c c c}
					\toprule
					&  WN18RR(v1)  & FB15k-237(v1) & NELL-995(v1) \\
					\midrule
					frequency-based & 76.30 & 20.10 & 47.00 \\
					TACT & {81.69} & {65.48} & {58.03} \\
                        TACO$_+$ & \textbf{95.12} & \textbf{75.93} & \textbf{60.81}\\
					\bottomrule
			         \end{tabular}}
			\label{tab:mrr-frequency}    
		\end{center}
	\end{table}

	\begin{table}[ht]
		\caption{The results for inductive link prediction of GraIL, TACT, and the proposed TACO$_+$ on the dataset YAGO3-10.}
		\begin{center}
			\resizebox{0.7\columnwidth}{!}{
				\begin{tabular}{c c c c}
					\toprule
					&  AUC-PR & MRR & Hits@1 \\
					\midrule
					GraIL & 0.634 & 0.158 & 0.048 \\
					TACT & {0.915} & {0.406} & {0.140} \\
                        TACO$_+$ & \textbf{0.930} & \textbf{0.471} &\textbf{0.184}\\
					\bottomrule
			             \end{tabular}}
			\label{tab:yago3}    
		\end{center}
	\end{table}

	\begin{table}[ht]
		\caption{The running time of GraIL, TACT, TACO, TACO$_+$, and NBFNet on the version 1 of the inductive datasets. We measure these methods on the same device for a fair comparison. }
		\begin{center}
			\resizebox{1.0\columnwidth}{!}{
				\begin{tabular}{c c c c}
					\toprule
					&  WN18RR(v1)  & FB15k-237(v1)  & NELL-995(v1)\\
					\midrule
					GraIL & 0.05 h & 0.11 h & 0.06 h  \\

					TACT & 0.07 h & 0.13 h & 0.09 h  \\
                        TACO &0.08 h &0.15 h &0.10 h \\
                        TACO$_+$ & 0.11 h & 0.18 h & 0.12 h\\
                        NBFNet &0.24 h & 0.27 h  & 0.19 h\\
					\bottomrule

			         \end{tabular}}
	
   \label{tab:running_time}
		\end{center}
	\end{table}

\subsubsection{Ablation on the input embeddings} \label{ab: input emb}
    In the proposed method, the input of scoring network is 
	\begin{align*}
		\textbf{r}_t^{F} \oplus \textbf{e}_{\mathcal{G}(u,r_t,v)}^{(L)} \oplus \textbf{e}_u^{(L)} \oplus \textbf{e}_v^{(L)}     
	\end{align*}
 
    That is, the score of the target predicted triple is the combination of the final relation embedding $\textbf{r}_t^{F}$, the graph embedding $\textbf{e}_{\mathcal{G}(u,r_t,v)}^{(L)}$, and the node embedding $\textbf{e}_u^{(L)} \oplus \textbf{e}_v^{(L)}$.

We conduct ablation experiments to get access to the exact effect of each part of embeddings in the inductive link prediction. Table \ref{Tab: ablation-2} shows the results of scoring a triple based on different combinations of embeddings. We can see that any part of the embeddings serves its own distinct effect on the final performance results. When performing inductive relation prediction, using the final relation embedding solely---which is exactly the method of our proposed baseline model TACO$_+$-base---can get a fairly good performance. This shows that modeling the semantic correlation between relations is beneficial to make the correct relation prediction in the inductive setting. The various embedding combinations will further promote the performance on different datasets. We can set the used embedding combination properly to further improve the performance of TACO to get the best results on different benchmark datasets for inductive link prediction.

  \subsubsection{Ablation on Different Relation Correlations}\label{ab: diff rc}
  In the proposed method, we take into account all different relation correlations of TACO$_+$. We conduct ablation experiments to figure out the influence on different relation correlations on TACO$_+$.
 By reducing the number of relation correlations and analyzing the impact of each correlation on inductive link prediction, we aim to investigate the effect of relation correlations in detail. As shown in Table \ref{Tab: ablation-NUM}, although some datasets exhibit improved performances when considering fewer relation correlation patterns, taking into account all relation correlation patterns leads to more stable and generally acceptable results. In specific datasets, omitting a particular relation correlation pattern, such as the "H-H" pattern on the FB15K-237-v1 dataset, results in significant performance degradation. Therefore, including all relation correlation patterns in the analysis yields more stable and robust results.

	\subsection{Further Experiments}\label{sec:further}
	
	\subsubsection{Case Studies}
	
	We select some relations and show the top three relevant relations of them in Table \ref{tab:ex2}. Recall that the sum of correlation coefficients for each correlation pattern is equal to 1. The results show that TACO can learn some correct correlation patterns and assign them high correlation coefficients. For example, among all the neighboring relations of ``{\_member\_meronym}", ``{\_has\_part}"---which is adjacent to ``{\_member\_meronym}" in the topological pattern of parallel---gets the most significant correlation coefficient, as ``{\_has\_part}" and ``{\_member\_meronym}" have similar semantics.  Notably, this semantic results are also human-understandable, which highlights the practical interpretability of TACO.

 \subsubsection{Comparison of the 3-hop enclosing subgraph}
 As mentioned in Section \ref{sec:saen}, the $3$-hop enclosing sugraph suffers from eliminating the irrelevant rules, which may cause the model to overfit to the extracted irrelevant rules within the subgraph and hinder the model performance.
 To further demonstrate the effectiveness of eliminating irrelevant rules of the proposed complete subgraph, we compare between these methods. As Table \ref{Tab: comparison} shows, the $2$-hop CCN and CCN+ TACO outperforms the $3$-hop TACT consistently and significantly, which demonstrates the effectiveness of CCN subgraph {methods}.

 \subsubsection{The Frequency-based Method}

  	We conduct an experiment by ranking relations according to their frequencies and compare TACO$_+$ with the frequency-based method on the datasets. For the frequency-based method, the returned rank list for every prediction is the same, which is the rank according to the relation frequencies from high to low in the knowledge graph. In other words, the frequency-based method represents a type of data bias in the datasets. As illustrated in Table \ref{tab:mrr-frequency}, TACO$_+$ and TACT significantly outperform the frequency-based method by a significant margin on the datasets. The results demonstrate that the effectiveness of TACO$_+$ is not due to the data bias of relation frequencies in benchmark datasets.

	\subsubsection{Results on YAGO3-10}
	
	To demonstrate the effectiveness of our proposed method on a larger knowledge graph with few relations. We conduct experiments on YAGO3-10, which is a subset of YAGO3 \cite{yago3} and contains 37 relations and 123,182 entities. Table \ref{tab:yago3} shows the results for inductive link prediction of GraIL, TACT, and TACO$_+$ on YAGO3-10. As we can see, TACO$_+$ outperforms GraIL and TACT by all the metrics, which demonstrate our proposed method can effectively deal with a larger knowledge graph with few relations and Complete Common Neighbor induced subgraph also helps to improve TACO by preserving more relevant rules within the extracted subgraphs.
	
	\subsubsection{Running Time} \label{sec:running}
	
 

Table \ref{tab:running_time} shows the running time of GraIL, TACT, TACO, and NBFNet. We can observe that TACO and TACT require more time than GraIL to model the correlations between relations but the additional computational cost is insignificant compared to the performance improvement. In contrast, NBFNet has the longest running time, as it reasons on the entire graph for each target link.

	\section{Conclusion}

     In this paper, we propose a novel inductive reasoning approach called TACO, which effectively unifies graph-level information and edge-level
    interactions in knowledge graphs. Specifically, we prove that correlations between any two relations can be categorized into seven topological patterns and convert the original knowledge graph into RCG. Based on RCG, we then propose RCN to learn the importance of the different patterns for inductive link prediction. To further promote the performance of TACO, we propose CCN subgraph that can preserve complete relevant relations for RCG, i.e., complete topological patterns for RCN. Extensive experiments demonstrate that TACO significantly outperforms existing state-of-the-art methods on benchmark datasets for the inductive link prediction task. 

\section*{Acknowledgment}


The authors would like to thank all the anonymous reviewers for their insightful comments. This work was supported in part by National Natural Science Foundations of China grants U19B2026, U19B2044, 61836011, 62021001, 61836006, and 2022ZD0119801.



%

\bibliographystyle{IEEEtran}
\bibliography{ieee_jrnl}

\begin{IEEEbiography}[{\includegraphics[width=1in,height=1.25in,clip,keepaspectratio]{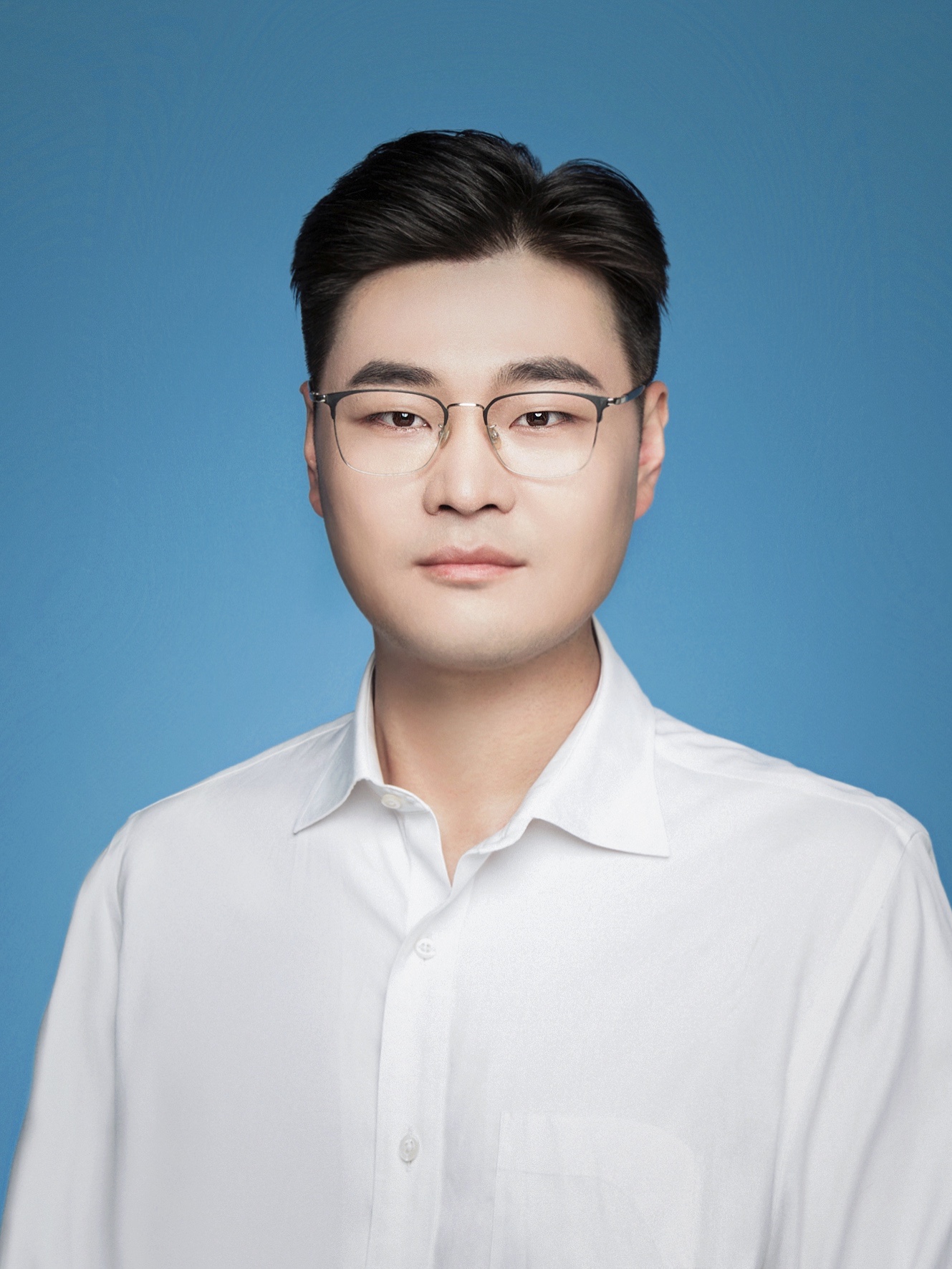}}]{Jie Wang}
  received the B.Sc. degree in electronic information science and technology from University of Science and Technology of China, Hefei, China, in 2005, and the Ph.D. degree in computational science from the Florida \mbox{State} University, Tallahassee, FL, in 2011. He is currently a professor in the Department of Electronic Engineering and Information Science at University of Science and Technology of China, Hefei, China. His research interests include AI for science, knowledge graph, large-scale optimization, deep learning, etc.  He is a senior member of IEEE.
\end{IEEEbiography}

\begin{IEEEbiography}[{\includegraphics[width=1in,height=1.25in,clip,keepaspectratio]{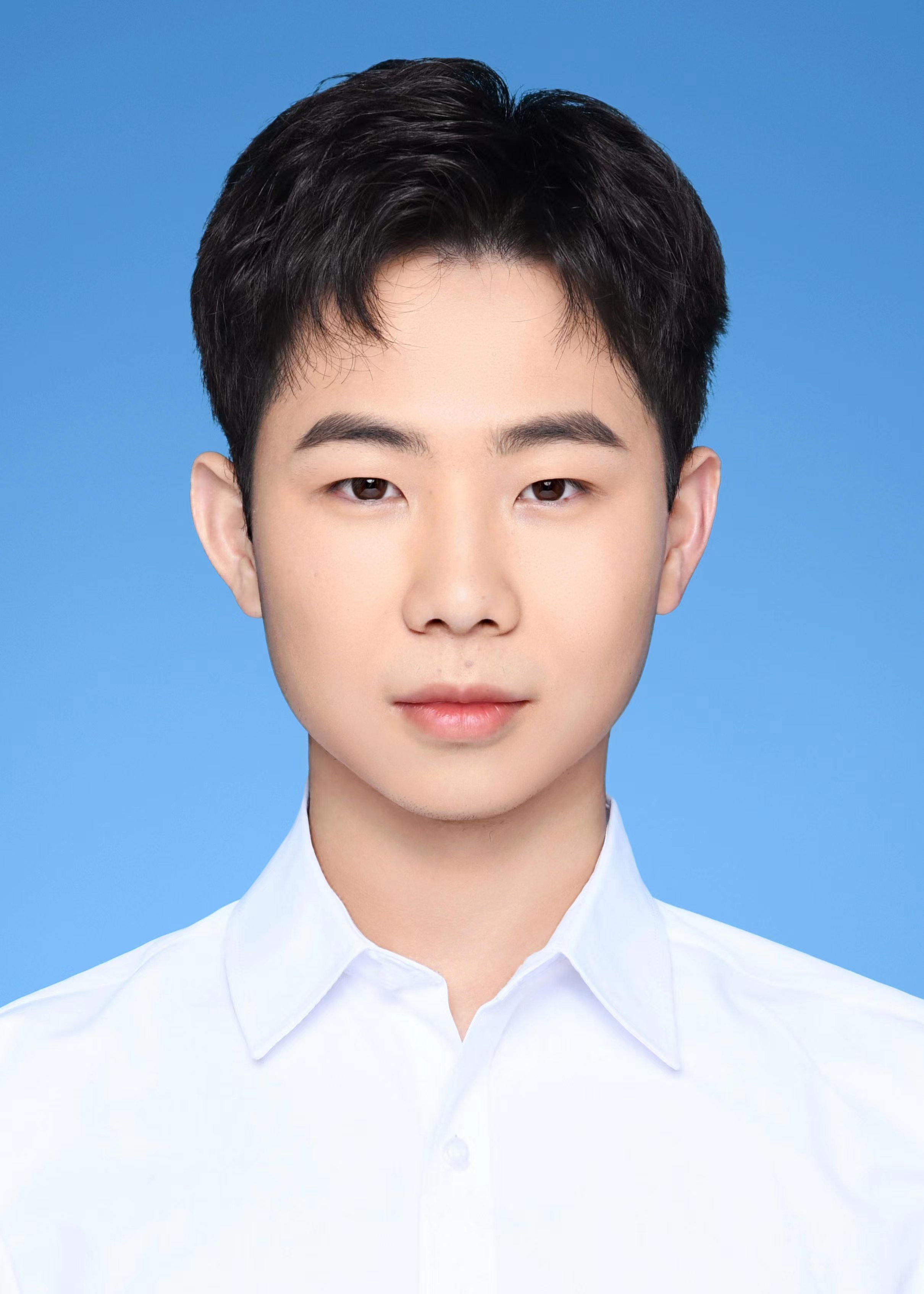}}]{Hanzhu Chen}
  received the B.Sc. degree in Computer Science and Technology from Southwest University, Chongqing, China, in 2021. He is currently a graduate student in the School of Data Science at University of Science and Technology of China, Hefei, China. His research interests include graph representation learning and natural language processing.
\end{IEEEbiography}

\begin{IEEEbiography}[{\includegraphics[width=1in,height=1.25in,clip,keepaspectratio]{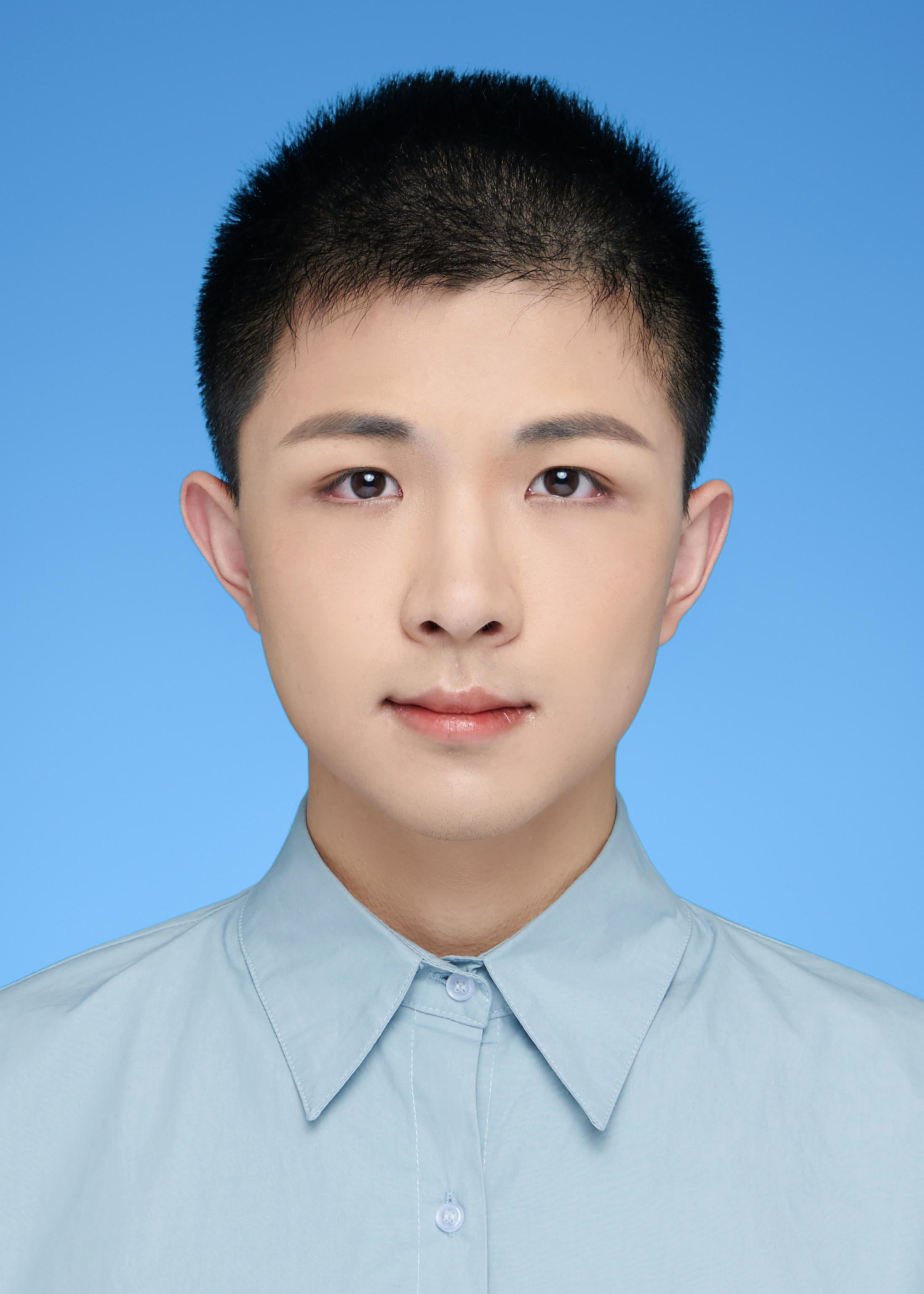}}]{Qitan Lv}
  received the B.Sc. degree in electronic and information engineering from South China University of Technology, GuangZhou, China, in 2023. He is currently a graduate student in the Department of Electronic Engineering and Information Science at University of Science and Technology of China, Hefei, China. His research interests include graph representation learning and natural language processing.
\end{IEEEbiography}

\begin{IEEEbiography}[{\includegraphics[width=1in,height=1.25in,clip,keepaspectratio]{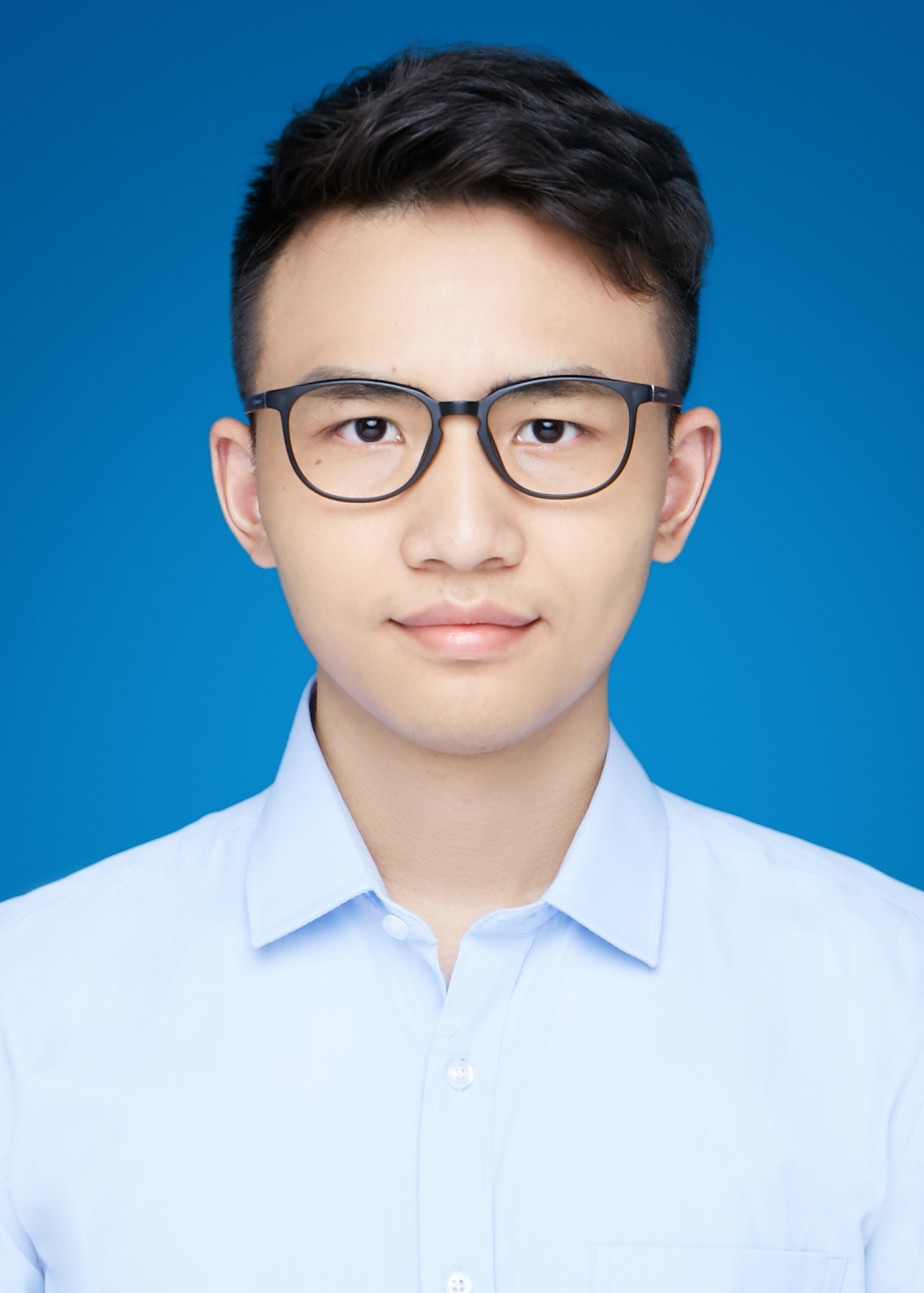}}]{Zhihao Shi}
  received the B.Sc. degree in Department of Electronic Engineering and Information Science from University of Science and Technology of China, Hefei, China, in 2020. a Ph.D. candidate in the Department of Electronic Engineering and Information Science at University of Science and Technology of China, Hefei, China. His research interests include graph representation learning and natural language processing.
\end{IEEEbiography}

\begin{IEEEbiography}[{\includegraphics[width=1in,height=1.25in,clip,keepaspectratio]{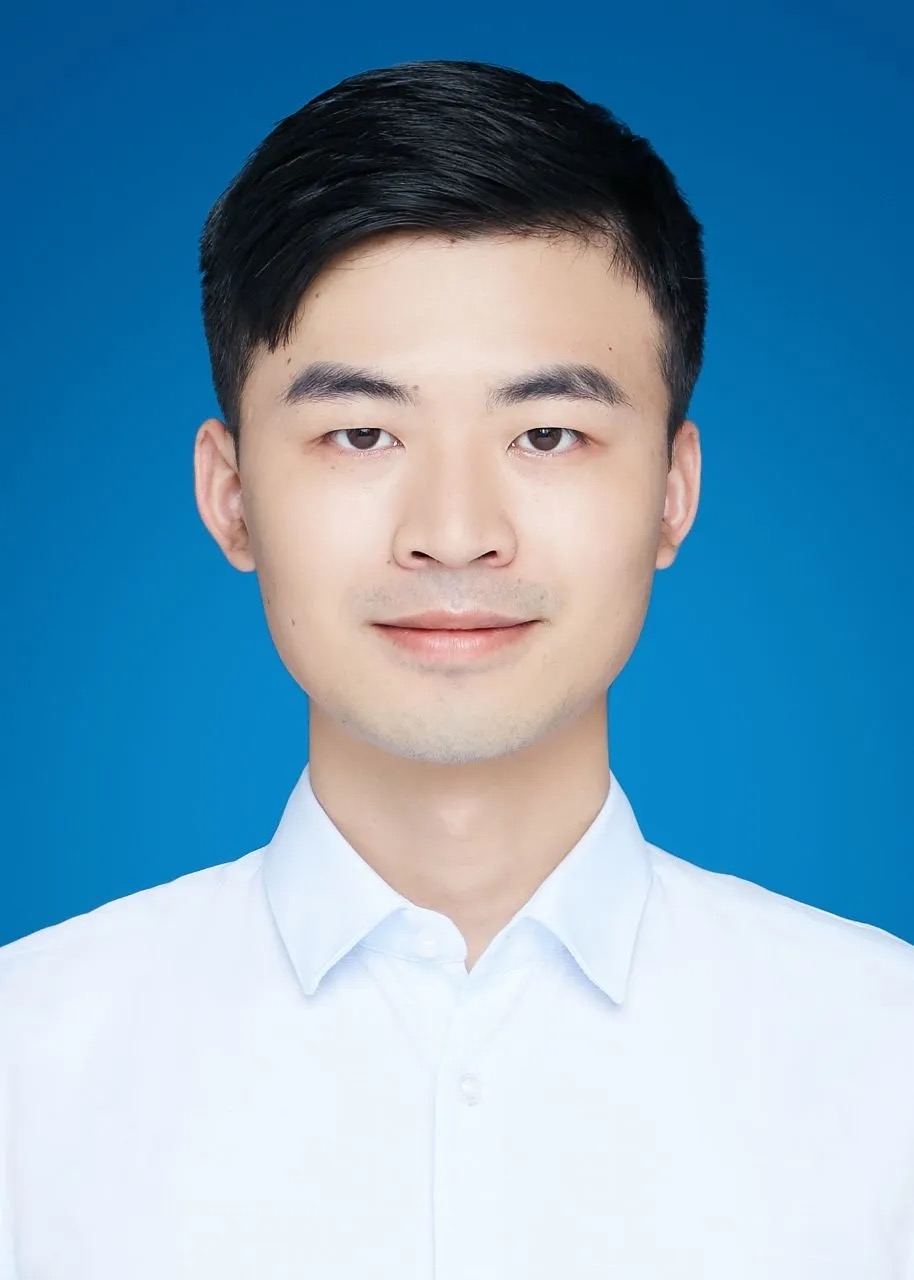}}]{Jiajun Chen}
 received the M.S. degree in Department of Electronic Engineering and Information Science from University of Science and Technology of China, Hefei, China, in 2022. His research interests include graph representation learning and natural language processing.
\end{IEEEbiography}

\begin{IEEEbiography}[{\includegraphics[width=1in,height=1.25in,clip,keepaspectratio]{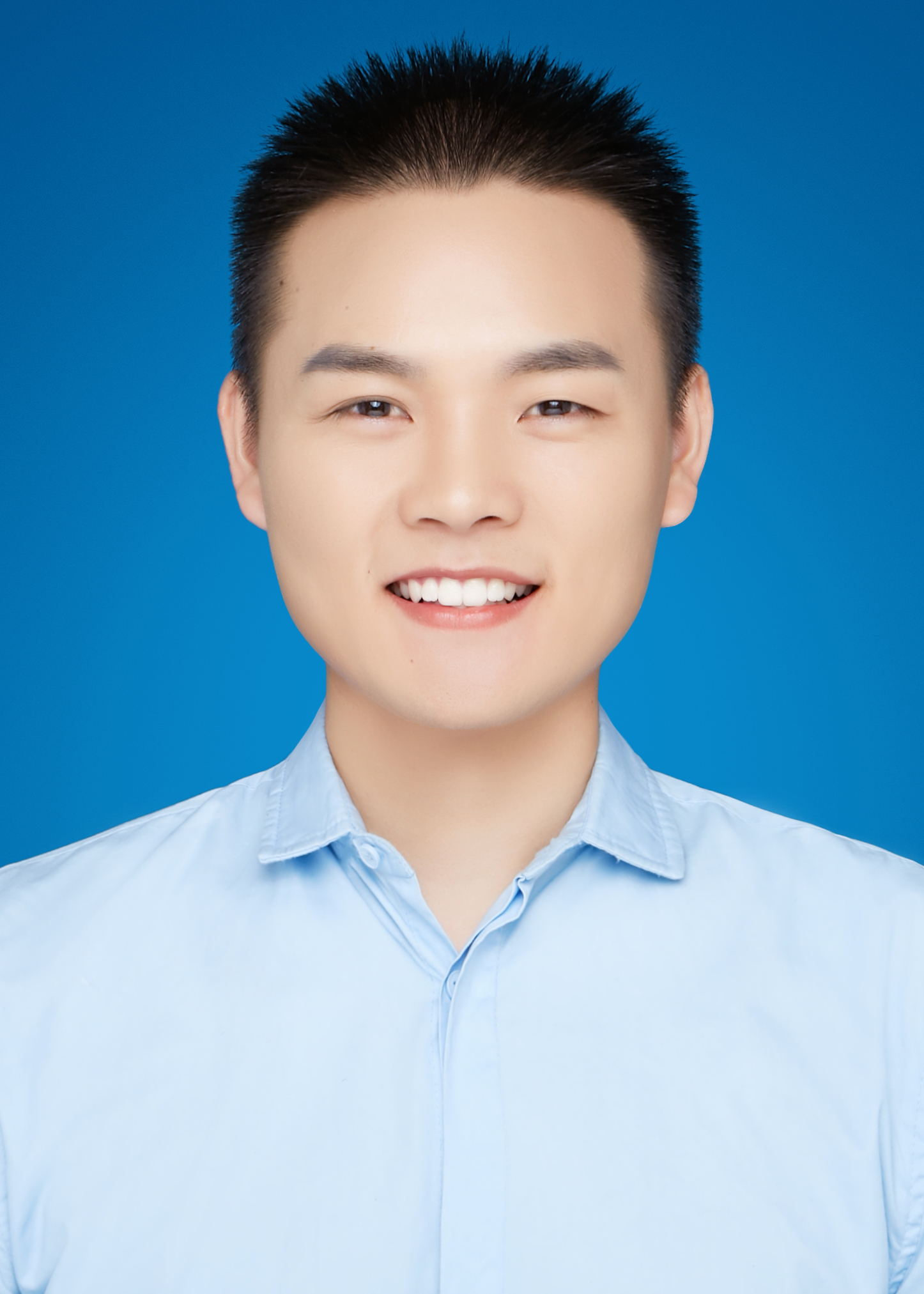}}]{Huarui He}
  received the B.Sc. degree in electronic engineering from Xidian University, Xi'an, China, in 2020. He is currently a graduate student in the Department of Electronic Engineering and Information Science at University of Science and Technology of China, Hefei, China. His research interests include graph representation learning and natural language processing.
\end{IEEEbiography}


\begin{IEEEbiography}[{\includegraphics[width=1in,height=1.25in,clip,keepaspectratio]{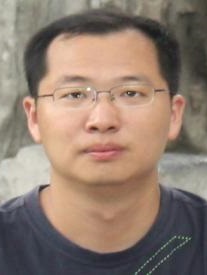}}]{Hongtao Xie}
  received the PhD degree in computer application technology from the Institute of
Computing Technology, Chinese Academy of Sciences, Beijing, China, in 2012. He is currently a
professor with the School of Information Science
and Technology, University of Science and Technology of China, Hefei, China. His research interests include multimedia content analysis and
retrieval, deep learning, and computer vision.
\end{IEEEbiography}

\begin{IEEEbiography}[{\includegraphics[width=1in,height=1.25in,clip,keepaspectratio]{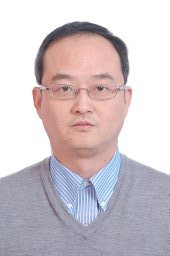}}]{Feng Wu}
received the B.S. degree in electrical engineering from Xidian University in 1992, and the M.S. and Ph.D. degrees in computer science from the Harbin Institute of Technology in 1996 and 1999, respectively. He is currently a Professor with the University of Science and Technology of China, where he is also the Dean of the School of Information Science and Technology. Before that, he was a Principal Researcher and the Research Manager with Microsoft Research Asia. His research interests include image and video compression, media communication, and media analysis and synthesis. He has authored or coauthored over 200 high quality articles (including several dozens of IEEE Transaction papers) and top conference papers on MOBICOM, SIGIR, CVPR, and ACM MM. He has 77 granted U.S. patents. His 15 techniques have been adopted into international video coding standards. As a coauthor, he received the Best Paper Award at 2009 IEEE Transactions on Circuits and Systems for Video Technology, PCM 2008, and SPIE VCIP 2007. He also received the Best Associate Editor Award from IEEE Circuits and Systems Society in 2012. He also serves as the TPC Chair for MMSP 2011, VCIP 2010, and PCM 2009, and the Special Sessions Chair for ICME 2010 and ISCAS 2013. He serves as an Associate Editor for IEEE Transactions on Circuits and Systems for Video Technology, IEEE Transactions ON Multimedia, and several other international journals.
\end{IEEEbiography}

\clearpage
\appendices 
\section{The Number of Topological Patterns} \label{appenA}
	\begin{theorem}
		In the knowledge graph, the number of topological patterns between any two irreflexive relations is at most seven.
	\end{theorem}
	\begin{proof}
		For any two edges $(h_1, t_1)$ and $(h_2, t_2)$ in the knowledge graph, the number of intersections between them can be $0,1,2$. Suppose the two edges are corresponding to two triples $(h_1, r_1, t_1)$ and $(h_2, r_2, t_2)$ in the knowledge graph. As the relations $r_1$ and $r_2$ both are irreflexive, we have $h_1\ne t_1$ and $h_2 \ne t_2$.
		\begin{enumerate}
			\item[(a)] If the number of intersections is $0$, we know the two edges are not connected, which implies that their topological structure only has one pattern. 
			\item[(b)] If the number of intersections is $1$, there are four cases: 
			\begin{enumerate}
				\item[1)] $h_1=h_2, t_1\ne t_2$;
				\item[2)] $h_1=t_2, h_2\ne t_1$;
				\item[3)] $t_1=h_2, h_1\ne t_2$;
				\item[4)] $t_1=t_2, h_1\ne h_2$.
			\end{enumerate}
			Each case is corresponding to a specific topological pattern. The topological structures can be head-to-head, head-to-tail, tail-to-head, and tail-to-tail, thus the number of topological patterns is four.
			\item[(c)] If the number of intersections is $2$, there are two cases:
			\begin{enumerate}
				\item[1)] $h_1=h_2, t_1=t_2$;
				\item[2)] $h_1=t_2, h_2=t_1$.
			\end{enumerate}
			Each case is corresponding to a specific topological pattern. The number of topological patterns is two.
		\end{enumerate}
		Therefore, the number of topological patterns between any two irreflexive relations is at most $1+4+2=7$.
	\end{proof}

 \section*{Proof of Common Neighbors in a Reasoning Path}

\subsection*{Part 1: Single-hop Neighbors}

Let $G = (V, E)$ be a graph with vertex set $V$ and edge set $E$. For any vertex $x$, its \textit{single-hop neighborhood}, denoted $N(x)$, is defined as the set of vertices directly connected to $x$, i.e., $N(x) = \{y \in V | (x, y) \in E \text{ or } (y, x) \in E\}$. A \textit{reasoning path} is a sequence of vertices $(x_1, x_2, \ldots, x_n)$ where $x_1 = u$, $x_n = v$, and $(x_i, x_{i+1}) \in E$ for all $1 \leq i < n$, denoted $P_{uv}$.

\begin{theorem} If two vertices $u, v \in V$ have a common single-hop neighbor $c$, then there exists a reasoning path $P_{uv}$ from $u$ to $v$ that includes $c$.
\end{theorem}

	\begin{proof}
Given $u, v \in V$ and a common single-hop neighbor $c \in N(u) \cap N(v)$, by definition, $(u, c) \in E$ and $(c, v) \in E$. The sequence of vertices $(u, c, v)$ forms a valid reasoning path $P_{uv}$, since it starts with $u$, ends with $v$, and each consecutive pair of vertices is connected by an edge in $E$. 
\end{proof}

\subsection*{Part 2: Multi-hop Neighbors}

For any vertex $x \in V$ and a positive integer $k$, the \textit{k-hop neighborhood} of $x$, denoted by $N_k(x)$, is defined as the set of vertices that can be reached from $x$ by traversing exactly $k$ edges. Formally, $N_k(x)$ includes vertex $y$ if there exists a sequence of vertices $(x = x_0, x_1, \ldots, x_k = y)$ such that $(x_{i-1}, x_i) \in E$ for all $1 \leq i \leq k$.

\begin{theorem} If two vertices $u, v \in V$ have a common multi-hop neighbor $c$, then there exists a reasoning path $P_{uv}$ from $u$ to $v$ that includes $c$.
\end{theorem}

	\begin{proof}
Let $u, v \in V$ be two vertices with a common multi-hop neighbor $c \in V$. This implies that for some positive integer $k$, $c \in N_k(u) \cap N_k(v)$. By the definition of $N_k$, there exist paths $P_{uc}$ from $u$ to $c$ and $P_{cv}$ from $c$ to $v$, each consisting of $k$ edges. 

To construct a reasoning path from $u$ to $v$ that includes the common multi-hop neighbor $c$, we concatenate the paths $P_{uc}$ and $P_{cv}$ where $c$ is the common vertex between the two paths. Formally, we can write:

\begin{equation}
P_{uv} = P_{uc} \oplus P_{cv}
\end{equation}

where $\oplus$ denotes the concatenation of the paths such that the common vertex $c$ appears only once in the concatenated path $P_{uv}$. Explicitly, if

\begin{align*}
P_{uc} &= (u, \ldots, c) \\
P_{cv} &= (c, \ldots, v)
\end{align*}

then the concatenated path $P_{uv}$ is given by:

\begin{align*}
P_{uv} &= (u, \ldots, c, \ldots, v)
\end{align*}

This results in a reasoning path that starts at $u$, passes through $c$, and ends at $v$, demonstrating that there exists such a path $P_{uv}$ that includes $c$. 

\end{proof}

\section{Implementation Details} \label{imp_details}
\subsection{Training Implementation}
We conduct all the experiments on an NVidia GeForce GTX 3090 GPU and an Intel(R) Xeon(R) Gold 6246R CPU @ 3.40GHz.  We use PyTorch\cite{paszke2019pytorch} and DGL\cite{dgl} to implement our TACO. For optimization, we use Adam optimizer \cite{adam} with batch size 16 and an initial learning rate of $1\times 10^{-2}$.
We divide the learning rate by 5 when the validation loss does not improve for 5 epochs, and stop training if it does not improve for 8. We compare the marginal ranking loss and binary cross entropy loss when training our TACO and apply the marginal ranking loss to construct TACO. The maximum training epoch of TACO is set to 20. When training our TACO-base model, we use an initial learning rate of $ 5 \times 10^{-2}$ with the same learning rate scheduler as TACO but just train maximum 15 epoches to get the optimum results. We will release our code once the paper is accepted to be published.

\subsection{Model Framework} 
First, we apply a graph extractor for each target link to get CCN subgraphs and corresponding RCGs. Then, we apply a two-layer R-GCN to reason on extracted CCN subgraphs and a two-layer GCN to reason on their corresponding relation correlation graphs. Finally, we apply the combination of CCN subgraphs' embedding vectors and relation correlation graphs' embedding vectors into a single layer perceptron to produce a plausibility score of the target prediction relation link and other negative candidate sample links for both classification and rank tasks.

\subsection{Hyperparameters}
Throughout all the experiments, we set the training and validation batch size for branching models to be 16.  We conduct grid search to obtain optimal hyperparameters, where we search dropout rate in \{0, 0.1, 0.2\}, edge-dropout rate in \{0, 0.3, 0.5\} and margins in the loss function in \{8, 10, 12, 16\}. The regularization coefficient of GNN weights is set to 0.01. Configuration for the best performance of each dataset is given within the code.


\ifCLASSOPTIONcaptionsoff
  \newpage
\fi

\end{document}